\documentclass[11pt]{article}
\pdfoutput=1
\usepackage{booktabs}
\usepackage{amsmath, amsthm}
\usepackage{amssymb}
\usepackage{graphicx}
\usepackage{setspace}
\usepackage{xfrac}
\usepackage{hyperref}
\usepackage{xstring}
\usepackage{xspace}

\usepackage{microtype}
\usepackage{graphicx}
\usepackage{subfigure}
\usepackage{booktabs} 

\usepackage{comment}

\usepackage{hyperref}       
\usepackage{amsfonts}       
\usepackage{nicefrac}       
\usepackage{microtype}      
\usepackage{amsmath}
\usepackage{amsfonts}
\usepackage{amssymb}
\usepackage{comment}
\usepackage[ruled,vlined]{algorithm2e}

\usepackage{hyperref}

\usepackage{amsthm}
\usepackage{amscd}
\usepackage{t1enc}
\usepackage{enumerate}
\usepackage{enumitem}
\usepackage[mathscr]{eucal}
\usepackage{indentfirst}
\usepackage{listings}
\usepackage{graphicx}
\usepackage{graphics}
\usepackage{pict2e}
\usepackage{epic}
\usepackage{epstopdf} 

\newenvironment{sproof}{%
  \proof}{\endproof}

\newcommand{\bE}{\mathbb{E}}
\newcommand{\bP}{\mathbb{P}}
\newcommand{\bI}{\mathbb{I}}
\newcommand{\hY}{\hat{Y}}
\newcommand{\err}{\mathrm{err}}
\newcommand{\disc}{\mathrm{disc}}

\renewcommand{\b}{\mathbb}
\renewcommand{\P}{\mathbf{P}}
\newcommand{\Q}{\mathbf{Q}}
\newcommand{\A}{\mathcal{A}}
\newcommand{\HC}{\mathcal{H}}

\renewcommand{\tilde}{\widetilde}
\renewcommand{\nu}{\vartheta}
\newcommand{\abs}[1]{\left| #1 \right|}
\newcommand{\sset}[1]{\left\{#1\right\}}

\allowdisplaybreaks
\usepackage{fullpage}
\newtheorem{theorem}{Theorem}
\newtheorem*{theorem*}{Theorem}

\newtheorem{lemma}{Lemma}

\newtheorem{proposition}{Proposition}
\newtheorem*{proposition*}{Proposition}
\theoremstyle{definition}

\newtheorem{definition}{Definition}
\newtheorem{example}[theorem]{Example}

\newtheorem*{remark*}{Remark}
\date{}
\title{\textbf{Fair Learning with Private Demographic Data}}
\author{Hussein Mozannar \thanks{Massachusetts Institute of Technology. Email: \texttt{mozannar@mit.edu}} \and Mesrob I. Ohannessian \thanks{University of Illinois at Chicago. Email: \texttt{mesrob@uic.edu}} \and Nathan Srebro \thanks{Toyota Technological Institute at Chicago. Email: \texttt{nati@ttic.edu}}}
\begin{document}
\maketitle
\begin{abstract}
Sensitive attributes such as race are rarely available to learners in real world settings as their collection is often restricted by laws and regulations. We give a scheme that allows individuals to release their sensitive information privately while still allowing any downstream entity to learn non-discriminatory predictors. We show how to adapt non-discriminatory learners to work with privatized protected attributes giving theoretical guarantees on performance. Finally, we highlight how the methodology could apply to learning fair predictors in settings where protected attributes are only available for a subset of the data.
\end{abstract}

\section{Introduction}
As algorithmic systems driven by machine learning start to play an increasingly important role in society, concerns arise over their compliance with laws, regulations and societal norms. In particular, machine learning systems have been found to be discriminating against certain demographic  groups in applications of criminal assessment, lending and facial recognition \cite{barocas-hardt-narayanan}. To ensure non-discrimination in learning tasks, knowledge of the sensitive attributes is essential, however, laws and regulation often prohibit access and use of this sensitive data. As an example, credit card companies do not have the right to ask about an individual's race when applying for credit, while at the same time they have to prove that their decisions are non-discriminatory \cite{credit_rights,chen2019fairness}.

Apple Card, a credit card offered by Apple and Goldman Sachs, was recently accused of being discriminatory  \cite{apple_card}. Married couples rushed to Twitter to report that there were significant differences in the credit limit given individually to each of them even though they had shared finances and similar income levels. Supposing Apple was trying to make sure its learned model was non discriminatory, it would have been forced to use proxies for gender and recent work has shown that proxies can be problematic by potentially underestimating discrimination \cite{kallus2019assessing}.
We are then faced with what seems to be two opposing societal notions to satisfy: we want our system to be non-discriminatory while maintaining the privacy of our sensitive attributes. Note that even if the features that our model uses are independent of the sensitive attributes, it is not enough to guarantee notions of non-discrimination that further condition on the truth, e.g. equalized odds. One potential workaround to this problem, ignoring legal feasibility, is to allow the individuals to release their data in a locally differentially private manner \cite{dwork2006calibrating} and then try to learn from this privatized data a non-discriminatory predictor.
This allows us to guarantee that our decisions are fair while maintaining a degree of individual privacy to each user.

In this work, we  consider a binary classification framework where we have access to non-sensitive features $X$ and locally-private versions of the sensitive attributes $A$ denoted by $Z$. The details of the problem formulation are given in Section \ref{sec:problem}. Our contributions are as follows:

\begin{itemize}[leftmargin=*]
    \item We first give sufficient conditions on our predictor for non-discrimination to be equivalent under $A$ and $Z$ and derive estimators to measure discrimination using the private attributes $Z$. (Section \ref{sec:auditing})
    \item We give a learning algorithm based on the two-step procedure of \cite{woodworth2017learning} and provide statistical guarantees for both the error and discrimination of the resulting predictor. The main innovation in terms of both the algorithm and its analysis is in accessing properties of the sensitive attribute $A$ by carefully inverting the sample statistics of the private attributes $Z$. (Section \ref{sec:learning})
    \item We highlight how some of the same approach can handle other forms of deficiency in demographic information, by giving an auditing algorithm with guarantees, when protected attributes are available only for a subset of the data. (Section \ref{sec:missing})
\end{itemize} 

Beyond the original motivation, this work conveys additional insight on the subtle trade-offs between error and discrimination. In this perspective, privacy is not in itself a requirement, but rather an analytic tool. We give some experimental illustrations of these trade-offs.
\section{Related Work}\label{sec:related}
Enforcing non-discrimination constraints in supervised learning has been extensively explored  with many algorithms proposed to learn fair predictors with methods that fall generally in one category among pre-processing \cite{zemel2013learning}, in-processing \cite{cotter2018training,agarwal2018reductions}, or post-processing \cite{hardt2016equality}. In this work we focus on group-wise statistical notions of discrimination, setting aside critical concerns of individual fairness \cite{dwork2012fairness}.

\cite{kilbertus2018blind} were the first to propose to learn a fair predictor without disclosing information about protected attributes, using secure multi-party computation (MPC). However, as \cite{jagielski2018differentially} noted, MPC does not guarantee that the predictor cannot leak individual information. In response, \cite{jagielski2018differentially} proposed differentially private (DP) \cite{dwork2006calibrating} variants of fair learning algorithms. More recent work have similarly explored learning fair and DP predictors \cite{cummings2019compatibility,xu2019achieving,alabi2019cost,bagdasaryan2019differential}. In our setting \emph{local} privacy maintains all the guarantees of DP in addition to not allowing the learner to know for certain any sensitive information about a particular data point.
Related work has also considered fair learning when the protected attribute is missing or noisy \cite{hashimoto2018fairness,gupta2018proxy,lamy2019noise,awasthi2019effectiveness,kallus2019assessing, wang2020robust}.

Among these, the most related setting is that of \cite{lamy2019noise}, but it has several critical contrasting points with the present work. The simplest difference is the generalization here to non-binary groups, and the corresponding precise characterization of the equivalence between exact non-discrimination with respect to the original and private attributes. More importantly, their approach is only the \emph{first} step of our algorithm. As we show in Lemma \ref{lemma:step1_guarantees}, the first step makes the non-discrimination guarantee depend on both the privacy level and the complexity of the hypothesis class, which could be very costly. We remedy this using the \emph{second} step of our algorithm. \cite{awasthi2019effectiveness} consider a more general noise model for the protected attributes in the training data, but assume access to the actual protected attributes at test time. The fact that at test time $A$ is provided guarantees that the predictor is not a function of $Z$ and hence for the LDP noise mechanism by Proposition \ref{prop:forzeqfora}, we know that it is enough to guarantee non-discrimination with respect to $Z$ to be non-discriminatory with respect to $A$, which considerably simplifies the problem.


\section{Problem Formulation} \label{sec:problem}

A predictor $\hat Y$ of a binary target $Y \in \{0,1\}$ is a function of non-sensitive attributes $X \in \mathcal{X}$ and possibly also of a sensitive (or protected) attribute $A \in \mathcal{A}$ denoted as $\hat{Y}:=h(X)$ or $\hat{Y}:=h(X,A)$. We consider a binary classification task where the goal is to learn such a predictor, while ensuring a specified notion of non-discrimination with respect to $A$. As an example, when deciding to extend credit to a given individual, the protected attribute could denote someone's race and sex and the features $X$ could contain the person's financial history, level of education and housing information. Note that $X$ could very well include proxies for $A$ such as zip code which could reliably infer race \cite{infer_race}. 

Our focus here is on statistical notions of  group-wise non-discrimination amongst which are the following:

\begin{definition}[Fairness Definitions]\label{def:fairness}
	A classifier $\hat{Y}$ satisfies:
	$\bullet~$ Equalized odds (\textrm{EO}) if \ $\forall a \in \mathcal{A}$
	\[
	\bP(\hat{Y}=1|A=a, Y=y) = \bP(\hat{Y}=1|Y=y) \ \ \forall y \in \{0,1\},
	\]
	$\bullet~$ Demographic parity (\textrm{DP})  if \ $\forall a \in \mathcal{A}$
	\[
	\bP(\hat{Y}=1|A=a) = \bP(\hat{Y}=1),
	\]
	$\bullet~$ Accuracy parity (\textrm{AP}) if \ $\forall a \in \mathcal{A}$
	\[
	\bP(\hat{Y}\neq Y|A=a) = \bP(\hat{Y}\neq Y), 
	\]
	$\bullet~$ False discovery ($\hat{y}=1$) / omission ($\hat{y}=0$) rates parity if \ $\forall a \in \mathcal{A}$
	\[
	\bP(\hat{Y}\neq Y|\hat{Y}= \hat{y}, A=a) =\bP(\hat{Y}\neq Y|\hat{Y}= \hat{y}).
	\]
\end{definition}
Our treatment extends to a very broad family of demographic fairness constraints,  let $ \mathcal{E}_1, \mathcal{E}_2$ be two probability events defined with respect to $(X,Y,\hat{Y})$, then define $( \mathcal{E}_1, \mathcal{E}_2)$-non-discrimination with respect to $A$  as having: 
	\begin{equation}
	\bP\left( \mathcal{E}_1| \mathcal{E}_2,A=a\right) = \bP\left( \mathcal{E}_1| \mathcal{E}_2,A=a'\right) \quad \forall a,a' \in \A
	\label{eq:general_fairness_def}
	\end{equation}
All the notions considered in Definition \ref{def:fairness} can be cast into the above formulation for one or more set of events $( \mathcal{E}_1, \mathcal{E}_2)$.
Additionally, one can naturally define approximate versions of the above fairness constraints. As an example, for the notion of equalized odds, let $\mathcal{A}=\{0,1,\cdots,|\mathcal{A}|-1\}$ and define $\gamma_{y,a}(\hat{Y})=\bP(\hat{Y}=1|Y=y,A=a)$, then $\hat{Y}$ satisfies $\alpha$-EO if:
\[
\max_{y \in \sset{0,1}, a \in \mathcal{A}} \Gamma_{ya}:= \abs{\gamma_{y,a}(\hat{Y}) - \gamma_{y,0}(\hat{Y})} \leq \alpha
\] 
While it is clear that learning or auditing fair predictors requires knowledge of the protected attributes, laws and regulations often restrict the use and the collection of this data \cite{jagielski2018differentially}. Moreover, even if there are no restrictions on the usage of the protected attribute, it is desirable that this information is not leaked by (1) the algorithm's output and (2) the data collected. Local differential privacy (LDP) guarantees that the entity holding the data does not know for certain the protected attribute of any data point, which in turn makes sure that any algorithm built on this data is differentially private. Formally
a locally $\epsilon-$differentially private mechanism $Q$ is defined as follows:

\begin{definition}
	$Q$ is $\epsilon-$differentially private if \cite{duchi2013local}:
	\[
	\max _{z, a, a^{\prime}} \frac{Q(Z=z | a)}{Q\left(Z=z | a^{\prime}\right)} \leq e^{\epsilon}
	\]
\end{definition}
The mechanism we employ is the randomized response mechanism \cite{warner1965randomized,kairouz2014extremal}:
\[
Q(z | a)=\left\{\begin{array}{ll}{\frac{e^{\varepsilon}}{|\mathcal{A}|-1+e^{\varepsilon}}:=\pi} & {\text { if } z=a} \\ {\frac{1}{|\mathcal{A}|-1+e^{\varepsilon}}:=\bar{\pi}} & {\text { if } z \neq a}\end{array}\right.
\]

The choice of the randomized response mechanism is motivated by its optimality for distribution estimation under LDP constraints \cite{kairouz2014extremal,kairouz2016discrete}

The hope is that LDP samples of $A$ are sufficient to ensure non-discrimination, allowing us to refrain from the problematic use proxies for $A$. For the remainder of this paper, we assume that we have access to $n$  samples $S = \{(x_i,y_i,z_i)\}_{i=1}^{n}$ which are the result of  an  $i.i.d$ draw from an unknown distribution $\mathbb{P}$ over $\mathcal{X} \times \mathcal{Y} \times \mathcal{A} $ where $\mathcal{A}=\{0,1,\cdots,|\mathcal{A}|-1\}$ and $\mathcal{Y}=\{0,1\}$, but where $A$ is not observed and instead $Z$ is sampled from $Q(.|A)$ independently from $X$ and $Y$. We call $Z$ the \emph{privatized protected attribute}. To emphasize the difference between $A$ and $Z$ with respect to fairness, let $q_{y,a}(\hat{Y})=\bP(\hat{Y}=1|Y=y,Z=a)$, note that $\hat{Y}$ satisfies $\alpha$-EO \emph{with respect to $Z$} if:
\[
\max_{y \in \sset{0,1}, a \in \mathcal{Z}} \abs{q_{y,a}(\hat{Y}) - q_{y,0}(\hat{Y})} \leq \alpha.
\] 
\section{Auditing for Discrimination}\label{sec:auditing}
The two main questions we answer in this section is whether non-discrimination with respect to $A$ and $Z$ are equivalent and how to estimate the non-discrimination of a given predictor.

First, note that if a certain predictor $\hat{Y}=h(X,Z)$ uses $Z$ for predictions and  is non-discriminatory with respect to $Z$, then it is possible for it to in fact be discriminatory with respect to $A$. In Appendix \ref{apx:proofs}, we give an explicit example of such a predictor, that violates the equivalence for EO.
This illustrates that na\"ive implementations of fair learning methods can be more discriminatory than perceived. Any method that na\"ively uses the attribute $Z$ for its final predictions cannot immediately guarantee any level of non-discrimination with respect to $A$ especially post-processing methods. 

This however is not the case when predictors do not avail themselves of the privatized protected attribute $Z$. Namely, let's consider $\hat{Y}$ that are only a function of $X$. Since the randomness in the privatization mechanism is independent of $X$, this implies in particular that $\hat{Y}$ is independent of $Z$ given $A$. 
Our first result is that exact non-discrimination is invariant under local privacy: 
\begin{proposition}
	\label{prop:forzeqfora}
	Consider any exact non-discrimination notion among equalized odds, demographic parity, accuracy parity, or equality of false discovery/omission rates. Let $\hat{Y}:=h(X)$ be a binary predictor, then $\hat{Y}$ is non-discriminatory with respect to $A$ if and only if it is non-discriminatory with respect to $Z$. 
\end{proposition}
\begin{sproof}
	We consider a general formulation of the constraints we previously mentioned, let $ \mathcal{E}_1, \mathcal{E}_2$ be two probability events defined with respect to $(X,Y,\hat{Y})$, then define non-discrimination with respect to $A$ as having: 
	\[
	\bP\left( \mathcal{E}_1| \mathcal{E}_2,A=a\right) = \bP\left( \mathcal{E}_1| \mathcal{E}_2,A=a'\right) \quad \forall a,a' \in \A
	\]
	Define this notion similarly with respect to $Z$. We can obtain the following relation for the conditional probabilities
	\begin{flalign}
	 &\bP\left( \mathcal{E}_1| \mathcal{E}_2,Z=a\right)  
	 	  \overset{}{=} \bP\left( \mathcal{E}_1| \mathcal{E}_2,A=a\right) \frac{\pi\bP(A=a,\mathcal{E}_2)}{\bP(Z=a,\mathcal{E}_2)} +  \sum_{a' \in \A \setminus \{a\}}  \bP\left( \mathcal{E}_1| \mathcal{E}_2,A=a'\right) \frac{\bar\pi\bP(A=a',\mathcal{E}_2)}{\bP(Z=a,\mathcal{E}_2)}   
	\end{flalign}
	
	 Let $P$ be the following $|\A| \times |\A|$ matrix:
\begin{equation}
    	\begin{cases}
	P_{i,i} =\frac{\pi\bP(A=i,\mathcal{E}_2)}{\bP(Z=i,\mathcal{E}_2)}   \ \text{for} \ i\in \mathcal{A}\\
	P_{i,j} = \frac{\bar\pi\bP(A=j,\mathcal{E}_2)}{\bP(Z=i,\mathcal{E}_2)} \ \text{for} \ i,j\in \mathcal{A} \ \text{s.t.} i \neq j\\
	\end{cases}
	\label{eq:matrix_P_definition}
	\end{equation}
	Then we have the following linear system of equations:
\begin{align}
	&\begin{bmatrix} 
		\bP( \mathcal{E}_1| \mathcal{E}_2,Z=0) \\
		\vdots \\
	\bP( \mathcal{E}_1| \mathcal{E}_2,Z=|\A|-1)
	\end{bmatrix}
	= P 	\begin{bmatrix} 
		\bP( \mathcal{E}_1| \mathcal{E}_2,A=0) \\
		\vdots \\
	\bP( \mathcal{E}_1| \mathcal{E}_2,A=|\A|-1)
	\end{bmatrix} \label{eq:relation_matrix}
	\end{align}
The matrix $P$ is row-stochastic and invertible, from this linear system we can deduce that non-discrimination with respect to $Z$ and $A$ are equivalent; details are left to Appendix \ref{apx:proofs}. 
\end{sproof}
Note that while $\hat{Y}$ not being a function of $Z$ is a sufficient condition for the conclusion in Proposition \ref{prop:forzeqfora} to hold, the more general condition for EO is that $\hat{Y}$ is independent of Z given A and Y, however actualizing this condition beyond simply ignoring $Z$ at test time is unclear.
We next study how to measure non-discrimination from samples. Unfortunately, Proposition \ref{prop:forzeqfora} applies only in the population-limit. For example for the notion EO, despite what it seems to suggest, na\"ive sample $\alpha$-discrimination relative to $Z$ underestimates discrimination relative to $A$.
Interestingly however, for any of the considered fairness notions, we can recover the statistics of the population with respect to $A$ via a linear system of equations relating them to those of $Z$ as in \eqref{eq:relation_matrix}.
This is done by inverting the matrix $P$ defined in \eqref{eq:matrix_P_definition}, however more care is needed: to compute the matrix $P$ one needs to compute quantities involving the attribute $A$, which then all have to be related back to $Z$.
Using this relation, we derive an estimator for the discrimination of a predictor that does not suffer from the bias of the na\"ive approach.  First we set  key notations for the rest of the paper:
$\P_{ya}:= \bP(Y=y,A=a)$, $\Q_{ya} := \bP(Y=y,Z=a) $ and $C= \frac{|\A| -2 + e^{\epsilon}}{e^{\epsilon} -1}$. The latter captures the scale of privatization: $C \approx O(\epsilon^{-1})$ if $\epsilon \ll 1$.

Let $P$ be the $\mathcal{A} \times \mathcal{A}$ matrix  as such:
	\[\begin{cases}
	P_{i,i} = \pi \frac{\P_{yi}}{\Q_{yi} }  \ \text{for} \ i\in \mathcal{A}\\
	P_{i,j} = \bar{\pi} \frac{\P_{yj}}{\Q_{yi} }  \ \text{for} \ i,j\in \mathcal{A} \ \text{s.t.} i \neq j\\
	\end{cases}\]
Then one can relate $q_{y,.}$ and $\gamma_{y,.}$ via:
\begin{align*}
	\begin{bmatrix} 
		q_{y0} \\
		\vdots \\
		q_{y,|\mathcal{A}|-1}
	\end{bmatrix}
	&= P \begin{bmatrix} 
		\gamma_{y,0} \\
		\vdots \\
		\gamma_{y,|\mathcal{A}|-1}
	\end{bmatrix}
\end{align*}
And thus by inverting P we can recover $\gamma_{y,a}$, however,
the matrix $P$ involves estimating the probabilities $\bP(Y=y,A=a)$ which we do not have access to but can similarly recover by noting that:
\begin{flalign}
  \nonumber  \Q_{yz} 
         &= {\pi} \P_{yz} + \sum_{a \neq z} \bar{\pi} \P_{ya} 
\end{flalign}
Let the matrix $\Pi \in \mathbb{R}^{|\mathcal{A}|\times|\mathcal{A}|}$ be as follows $\Pi_{i,j} = \pi$ if $i=j$ and $\Pi_{i,j}=\bar{\pi}$ if $i \neq j$.
Therefore $\Pi^{-1}_k \Q_{y,.}  = \bP(Y=y,A=k)$ where $\Pi^{-1}_k$ is the $k$'th row of $\Pi^{-1}$. Hence we can plug this estimates to compute $P$ and invert the linear system to measure our discrimination.
In Lemma \ref{lemma:concentration_of_estimator}, we characterize the sample complexity needed by our estimator to bound the violation in discrimination, specifically for the EO constraint. The privacy penalty $C$ arises from $||P||_{\infty}$.
\begin{lemma}
	For any $\delta \in (0,1/2)$, any binary predictor $\hat{Y}:=h(X)$, denote by  $\tilde{\Gamma}_{ya}^S$ our proposed estimator for $\Gamma_{ya}$ based on $S$, if
	$n \geq \frac{8\log(8|\A|/\delta)}{\min_{ya}\P_{ya}}$, we have:
	\begin{equation*}
	\bP\left( \max_{ya}| \tilde{\Gamma}_{ya}^S - \Gamma_{ya}| >
	\ \sqrt{\frac{\log(16/\delta)}{2n}} \frac{4C^2}{\min_{ya}\P_{ya}^2}
	\right) \leq \delta
	\end{equation*}
	\label{lemma:concentration_of_estimator}
\end{lemma}

\section{Learning Fair Predictors}\label{sec:learning}
In this section, we give a strategy to learn a non-discriminatory predictor with respect to $A$ from the data $S$, which only contains the privatized attribute $Z$. As in Lemma \ref{lemma:concentration_of_estimator}, for concreteness and clarity we restrict the analysis to the notion of equalized odds (EO) --- most of the analysis extends directly to other constraints. In light of the limitation identified by Proposition \ref{prop:forzeqfora}, let $\mathcal{H}$ be a hypothesis class of functions that depend only on $X$. Instead of a single predictor in the class, we exhibit a distribution over hypotheses, which we interpret as a randomized predictor. Let $\Delta_{\HC}$ be the set of all distributions over $\mathcal{H}$, and denote such a randomized predictor by $Q\in \Delta_{\HC}$. The goal is to learn a predictor that approximates the performance of the optimal non-discriminatory distribution:
\begin{align}
&Y^{*} = \arg\min_{Q \in \Delta_{\HC}} \bP(Q(X) \neq Y) \quad \\&s.t. \quad \gamma_{y,a}(Q) = \gamma_{y,0}(Q)   \ \forall y \in \{0,1\}, \forall a \in \mathcal{A}
\end{align}
A first natural approach would be to treat the private attribute $Z$ as if it were $A$ and ensure on $S$ that the learned predictor is non-discriminatory. Since the hypothesis class $\mathcal{H}$ consists of functions that depend only on $X$, Proposition \ref{prop:forzeqfora} applies and offers hope that, if we are able to achieve exact non discriminatory with respect to $Z$, we would be in fact non-discriminatory with respect to $A$. There are two problems with the above approach. First, exact non-discrimination is computationally hard to achieve and  approximate non-discrimination underestimates the discrimination by the privacy penalty $C$. And second, using an in-processing method such as the reductions approach of \cite{agarwal2018reductions} to learn results in a discrimination guarantee that scales with the complexity of $\mathcal{H}$.

Our approach is to adapt the two-step procedure of \cite{woodworth2017learning} to our setting. We start by dividing our data set $S$ into two equal parts $S_1$ and $S_2$. The first step is to learn an approximately non-discriminatory predictor $\hat{Y}=Q(X)$ with respect to $Z$ on $S_1$ via the reductions approach of \cite{agarwal2018reductions} which we detail in the next subsection. This predictor has low error, but may be highly discriminatory due to the complexity of the class affecting the generalization of non-discrimination of $\hat{Y}$. The aim of the second step is to produce a final predictor $\tilde{Y}$ that corrects for this discrimination, without increasing its error by much. We modify the post-processing procedure of \cite{hardt2016equality} to give us non-discrimination with respect to $A$ directly for the derived predictor $\tilde{Y}=f(\hat{Y},Z)$. The predictor in the second step \emph{does use} $Z$,  however with a careful analysis we are able to show that it indeed guarantees non-discrimination with respect to $A$; note that naively using the post-processing procedure of \cite{hardt2016equality}  fails. Two relationships link the first step to the second: how discrimination with respect to $Z$ and with respect to $A$ relate and how the discrimination from the first step affects the error of the derived predictor. In the following subsections we describe each of the steps, along with the statistical guarantees on their performance.
\subsection{Step 1: Approximate Non-Discrimination with respect to Z}
The first step aims to learn a predictor $\hat{Y}$ that is approximately $\alpha_n$-discriminatory with respect to $Z$ defined as:
\begin{align}
\hat{Y} = &\arg \min_{Q \in \Delta_{\HC}} \err^{S_1}(Q(X)) \label{eq:step1_formulation} \\
&\textrm{s.t.} \ \max_{y \in \{ 0,1\}} |q^{S_1}_{y,a}(Q) - q^{S_1}_{y,a}(Q)|\leq \alpha_n \label{eq:step1_constraint}
\end{align}
where for $Q \in \Delta_{\HC}$, we use the shorthand $\err(Q) = \bP(Q(X) \neq Y)$ and
quantities with a superscript  $S$ indicate their empirical counterparts. To solve the optimization problem defined in \eqref{eq:step1_formulation}, we reduce the constrained optimization problem to a weighted unconstrained problem following the approach of \cite{agarwal2018reductions}. As is typical with the family of fairness criteria considered, the  constraint in \eqref{eq:step1_constraint} can be rewritten as a linear constraint on $\hat{Y}$ explicitly. Let  $\mathcal{J}= \mathcal{Y} \times \mathcal{A}$, $\mathcal{K}= \mathcal{Y} \times \mathcal{A}\setminus \{0\} \times \{-,+\}$ and define $\boldsymbol{\gamma}(Q)\in \mathbb{R}^{|\mathcal{J}|}$ with $\boldsymbol{\gamma}(Q)_{(y,a)}=\gamma_{y,a}(Q)$, with the matrix $M \in \mathbb{R}^{|\mathcal{K}| \times |\mathcal{J}|}$ having entries:
$M_{(y,a,+),(a',y')}= \bI(a=a',y=y'),
M_{(y,a,-),(a',y')}= -\bI(a=a',y=y'),
M_{(y,a,+),(0,y')}= \bI(y=y'),
M_{(y,a,-),(0,y')}= -\bI(y=y') $. With this reparametrization, we can write $\alpha_n$-EO as:
\begin{equation}
M \boldsymbol{\gamma}(Q) \leq \alpha_n \mathbf{1}
\end{equation}
Let us introduce the Lagrange multiplier $\boldsymbol\lambda \in \mathbb{R}^{|\mathcal{K}|}_+$ and define the Lagrangian:
\begin{equation}
L(Q,\boldsymbol\lambda) = \err(Q) + \boldsymbol\lambda ^\top (M\boldsymbol\gamma(Q) - \alpha \mathbf{1})
\end{equation}
We constrain the norm of $\boldsymbol\lambda $ with $B \in \mathbb{R}^+$ and consider the  following two dual problems:
\begin{flalign}
&\min_{Q \in \Delta_{\HC}} \max_{\boldsymbol\lambda \in \mathbb{R}^{|\mathcal{K}|}_+, ||\boldsymbol{\lambda}||_1\leq B} L(Q,\boldsymbol{\lambda}) \\
& \max_{\boldsymbol\lambda \in \mathbb{R}^{|\mathcal{K}|}_+, ||\boldsymbol{\lambda}||_1\leq B} \min_{Q \in \Delta_{\HC}} L(Q,\boldsymbol{\lambda})
\end{flalign}
Note that $L$ is linear in both $Q$ and $\boldsymbol{\lambda}$ and their domains are convex and compact, hence the respective solution of both problems form a saddle point of $L$ \cite{agarwal2018reductions}. To find the saddle point, we treat our problem as a zero-sum game between two players: the Q-player ``learner'' and the $\lambda$-player ``auditor'' and use the strategy of \cite{freund1996game}. The auditor follows the exponentiated gradient algorithm and the learner picks it's best response to the auditor. 
The approach is fully described in Algorithm \ref{alg:reductions}.

\vspace{0.5cm}

\begin{algorithm}[H]
	\DontPrintSemicolon 
	\SetAlgoLined
	Input: training data $(X_i,Y_i,Z_i)_{i=1}^{n/2}$, bound $B$, learning rate $\eta$, rounds $T$ \\
	$\boldsymbol\theta_1 \gets \mathbf{0} \in \mathbb{R}^{|\mathcal{K}|}$\\
	\For{$t=1,2,\cdots,T$}{
		$\boldsymbol\lambda_{t,k} \gets B \frac{\exp(\theta_{t,k})}{1+ \sum_{k'} \exp(\theta_{t,k})} \forall k \in \mathcal{K}$\\
		$h_t \gets \textrm{BEST}_h(\boldsymbol{\lambda}_t) $\\
		$ \boldsymbol\theta_{t+1} \gets \boldsymbol\theta_t + \eta (M\boldsymbol\gamma^S(h_t) - \alpha_n \mathbf{1} )$\\
	}
	$ \hat{Y} \gets \frac{1}{T} \sum_{t=1}^T h_t ,  \hat{\boldsymbol{\lambda}} \gets \frac{1}{T} \sum_{t=1}^T \boldsymbol{\lambda}_t$\\
	Return $(\hat{Y},\hat{\boldsymbol{\lambda}} )$ \\
	\caption{Exp. gradient reduction for fair classification \cite{agarwal2018reductions}}
	\label{alg:reductions}
\end{algorithm}

\vspace{0.5cm}

Faced with a given vector $\boldsymbol{\lambda}$ the learner's best response, $\textrm{BEST}_h(\boldsymbol{\lambda})$), puts all the mass on a single predictor $h \in \HC$ as the Lagrangian  $L$ is linear in $Q$. \cite{agarwal2018reductions} shows that finding the learner's best response amounts to solving a cost-sensitive classification problem. We reestablish the reduction in detail in Appendix \ref{apx:proofs}, as there are slight differences with our setup. In particular, in Lemma \ref{lemma:step1_guarantees}, we establish a generalization bound on the error of the first step predictor $\hat{Y}$ and on its discrimination, defined as the maximum violation in the EO constraint. To denote the latter similarly to the error, we use the shorthand $\text{disc}(\hat{Y})=\max_{y \in \sset{0,1}, a \in \mathcal{A}} \Gamma_{ya}$.

\begin{lemma}
	Given a hypothesis class $\mathcal{H}$, a distribution over $(X,A,Y)$, $B \in \mathbb{R}^+$ and any $\delta \in (0,1/2)$, then with probability greater than $1- \delta$, if  $n \geq \frac{16\log{8 |\A|/\delta}}{\min_{ya} \P_{ya}}$, $\alpha_n = 2 \sqrt{\frac{\log{64|\A|/\delta}}{n \min_{ya}\P_{ya}}}$ and we let $\nu  = \mathfrak{R}_{n/2}(\mathcal{H}) + \sqrt{\frac{\log{8/\delta}}{n}}$, then running Algorithm 1 on data set $S$ with $T \geq \frac{16 \log(4 |\mathcal{A}|+1)}{\nu^2}$ and learning rate $\eta = \frac{\nu}{8 B}$ returns a predictor $\hat{Y}$ satisfying the following:
	\begin{align*}
	&\err(\hat{Y}) \leq_{\delta/2} \err(Y^*)  + 4 \mathfrak{R}_{n/2}(\mathcal{H}) + 4 \sqrt{\frac{\log{8/\delta}}{n}} \\
	&\disc(\hat{Y}) \leq_{\delta/2}  \frac{5C}{\min_{ya} \mathbf{P}_{ya}^2}  \left( \frac{2}{B}  + 6 \mathfrak{R}_{\frac{\min_{ya}n \P_{ya}}{4}}(\mathcal{H}) \right.  \left. + 10 \sqrt{\frac{2\log{64 |\A|/\delta}}{n \min_{ya}\P_{ya}}} \right) \ \  \text{(discrimination guarantee)}
	\end{align*}
	
	\label{lemma:step1_guarantees}
\end{lemma}
Proof of Lemma \ref{lemma:step1_guarantees} can be found in Appendix \ref{apx:proofs}. Note that the error bound in Lemma \ref{lemma:step1_guarantees} does not scale with the privacy level, however the discrimination bound is not only hit by the privacy, through $C$, but is further multiplied by the Rademacher complexity $\mathfrak{R}_{n}(\mathcal{H})$ of $\HC$. Our goal in the next step is to reduce the sample complexity required to achieve low discrimination by removing the dependence on the complexity of the model class in the discrimination bound. 

\textbf{Comparison with Differentially Private predictor.} \cite{jagielski2018differentially} modifies Algorithm \ref{alg:reductions} to ensure that the model is differentially private with respect to $A$ assuming access to data with the non-private attribute $A$. The error and discrimination generalization bounds obtained (Theorem 4.4 \cite{jagielski2018differentially}) both scale with the privacy level $\epsilon$ and the complexity of $\HC$, meaning the excess terms in the bounds of Lemma \ref{lemma:step1_guarantees} are both  in the order of $O(\mathfrak{R}_{n}(\mathcal{H})/\epsilon)$ in their work. Contrast this with our error bound that is independent of $\epsilon$, the catch is that discrimination obtained with LDP is significantly more impacted by the privacy level $\epsilon$. Thus, central differential privacy and local differential privacy in this context give rise to a very different set of trade-offs.

\subsection{Step 2: Post-hoc Correction to Achieve Non-Discrimination for $A$}

We correct the predictor we learned in step $1$ using a modified version of the post-processing procedure of \cite{hardt2016equality} on the data set $S_2$.
The derived  second step predictor $\tilde{Y}$ is fully characterized by $2 |\mathcal{A}|$ probabilities $\bP(\tilde{Y}=1|\hat{Y}=\hat{y},Z=a):=p_{\hat{y},z}$. If we na\"ively derive the predictor applying the post-processing procedure of \cite{hardt2016equality} on $S_2$ then this \emph{does not} imply that the predictor satisfies EO as the derived predictor is an explicit function of $Z$, cf. the discussion in Section \ref{sec:auditing}. Our approach is to directly ensure non-discrimination with respect to $A$ to achieve our goal. Two facts make this possible. First, the base predictor of step $1$ is not a function of $Z$ and hence we can measure its false negative and positive rates using the estimator from Lemma \ref{lemma:concentration_of_estimator}. And second, to compute these rates for $\tilde{Y}$, we can exploit its special structure. In particular, note the following decomposition: 
\begin{flalign}
\bP(\tilde{Y}=1|Y=y,A=a) &= \label{eq:posthoc_decomposition}
\bP(\tilde{Y}=1|\hat{Y}=0,A=a) \bP(\hat{Y} =0 | Y=y, A =a)  \\&+ \bP(\tilde{Y}=1|\hat{Y}=1,A=a) \bP(\hat{Y} =1 | Y=y, A =a) \nonumber
\end{flalign}
and we have that: \[\bP(\tilde{Y}=1|\hat{Y}=\hat{y},A=a) = \pi p_{\hat{y},a} +  \bar{\pi} \sum_{a' \in \mathcal{A}\setminus a}  p_{\hat{y},a'}:=\tilde{p}_{\hat{y},a}\]
 and  $\bP(\hat{Y} | Y=y, A =a)$ can be recovered by Lemma \ref{lemma:concentration_of_estimator}, denote
 $\tilde{\bP}^{S_2}(\hat{Y} = \hat{y} | Y=y, A =a)$ our estimator based on the empirical $\bP^{S_2}(\hat{Y} | Y, Z)$. Therefore we can compute sample versions of the conditional probabilities \eqref{eq:posthoc_decomposition}.

Our modified post-hoc correction reduces to solving the following constrained linear program for $\tilde{Y}$:
\begin{flalign}
\tilde{Y} = \arg\min_{p_{.,.}} \quad   \nonumber &\sum_{\hat{y},a} \left(\tilde{\bP}^{S_2}(\hat{Y}=\hat{y}, Z=a, Y =0) \right. \left. - \tilde{\bP}^{S_2}(\hat{Y}=\hat{y}, Z=a, Y =1) \right) \cdot \tilde{p}_{\hat{y},a} 
\end{flalign}
\begin{flalign}
s.t.& \quad  \left| \tilde{p}_{0,a} \nonumber \tilde{\bP}^{S_2}(\hat{Y} = 0 | Y=y, A =a)\right. + \tilde{p}_{1,a}\tilde{\bP}^{S_2}(\hat{Y} = 1 | Y=y, A =a) \nonumber \\ & - \tilde{p}_{0,0}  \tilde{\bP}^{S_2}(\hat{Y} = 0 | Y=y, A =0)\left. - \tilde{p}_{1,0}  \tilde{\bP}^{S_2}(\hat{Y} = 1 | Y=y, A =0) \right| \le \tilde{\alpha}_n , \ \forall y,a \nonumber   \\
& 0 \leq p_{\hat{y},a} \leq 1 \quad \forall \hat{y} \in \{ 0,1\}, \forall a \in \mathcal{A}  
\end{flalign}

The following Theorem illustrates the performance of our proposed estimator $\tilde{Y}$.
\begin{theorem}
For any hypothesis class $\mathcal{H}$, any distribution over $(X,A,Y)$ and any $\delta \in (0,1/2)$, then with probability $1- \delta$,
if $n \ge \frac{16\log(8|\A|/\delta)}{\min_{ya}\P_{ya}}$, $\alpha_n =  \sqrt{\frac{8\log{64/\delta}}{n \min_{yz}\Q_{yz}}}$ and $\tilde{\alpha}_n =  \sqrt{\frac{\log(64/\delta)}{2n}} \frac{4 |\mathcal{A}|C^2}{\min_{ya}\P_{ya}^2}$ , the predictor resulting from the two-step procedure satisfies:
\begin{align*}
 &\err(\tilde{Y}) \nonumber    \leq_{\delta} \err(Y^*)   +  \frac{5C}{\min_{ya} \mathbf{P}_{ya}^2}  \left( \frac{2}{B}  + 10 \mathfrak{R}_{\frac{\min_{ya}n \P_{ya}}{4}}(\mathcal{H}) \right. \left.+ 18 |\mathcal{A}| \sqrt{\frac{2\log{64 |\A|/\delta}}{n \min_{ya}\P_{ya}}} ~\right)  \nonumber \\
   &\disc(\tilde{Y}) \le_{\delta}  \sqrt{\frac{ \log(\frac{64}{\delta})}{2n}} \frac{8 |\mathcal{A}|C^2}{\min_{ya}\P_{ya}^2} &
\end{align*}

\end{theorem}
\begin{sproof}
 Since the predictor obtained in step 1 is only a function of $X$, we can prove the following guarantees on its performance with $\tilde{Y}^*$ being an optimal non-discriminatory derived predictor from $\hat{Y}$:
\begin{align*}
& \err(\tilde{Y})    \leq_{\delta/2} \err(\tilde{Y}^*) +   4 |\mathcal{A}|C \sqrt{\frac{\log(32 |\mathcal{A}|/\delta)}{2n}}  \nonumber \\
   &\disc(\tilde{Y}) \le_{\delta/2}  \sqrt{\frac{ \log(\frac{64}{\delta})}{2n}} \frac{8 |\mathcal{A}|C^2}{\min_{ya}\P_{ya}^2} 
\end{align*}
We next have to relate the loss of the optimal derived predictor from $\hat{Y}$, denoted by $\tilde{Y}^*$, to the loss of the optimal non-discriminatory predictor in $\mathcal{H}$. We can apply Lemma 4 in \cite{woodworth2017learning} as the solution of our derived LP is in expectation equal to that in terms of $A$. Lemma 4 in \cite{woodworth2017learning} tells us that the optimal derived predictor has a loss that is less or equal than the sum of the loss of the base predictor and it s discrimination:
\begin{equation}
\err(\tilde{Y}^*) \leq \err(\hat{Y}) + \disc(\hat{Y}) \nonumber 
\end{equation}
Plugging in the error and discriminating proved in Lemma \ref{lemma:step1_guarantees} we obtain the theorem statement. A detailed proof is given in Appendix \ref{app:step-2}.
\end{sproof}

Our final predictor $\tilde{Y}$ has a discrimination guarantee that is independent of the model complexity, however this comes at a cost of a privacy penalty entering the error bound. This creates a new set of trade-offs that do not appear in the absence of the privacy constraint, fairness and error start to trade-off more severely with increasing levels of privacy.

\subsection{Experimental Illustration}
\begin{figure}
    \centering
    \includegraphics[scale=0.8,trim={0cm 0cm 2cm 2cm}]{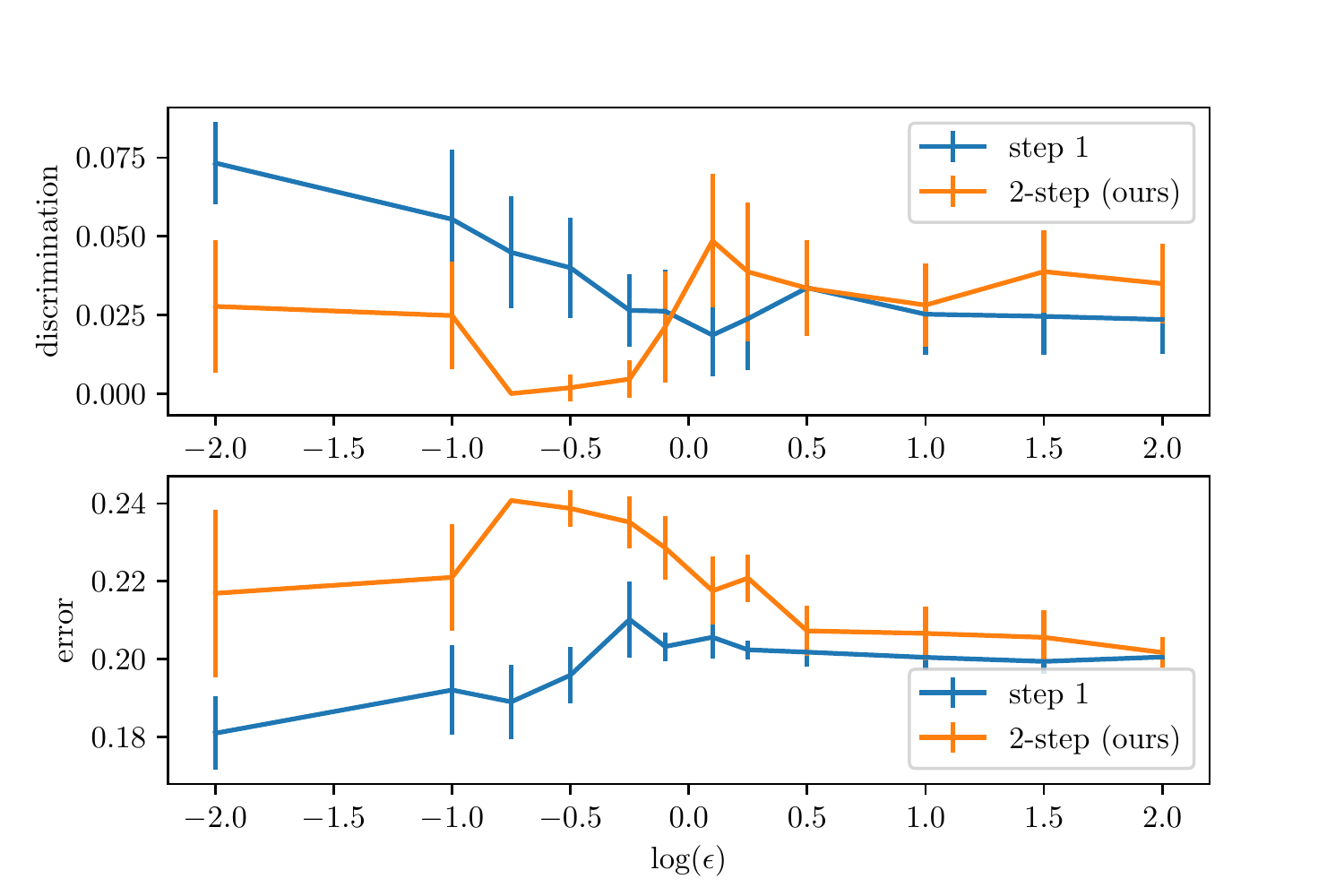}
    \caption{Plots of discrimination violation and accuracy of the step 1 predictor $\hat{Y}$ and the two-step predictor $\tilde{Y}$ versus the privacy level $\epsilon$ on the Adult Income dataset \cite{kohavi1996scaling}. Error bars show 95\% confidence interval for the average.}
    \label{fig:exp_adult}
\end{figure}
\textbf{Data.} We use the adult income data set \cite{kohavi1996scaling} containing 48,842
examples. The task is to predict whether a person's income is higher than $\$50k$. Each data point has 14 features including education and occupation, the protected attribute $A$ we use is gender: male or female.

\textbf{Approach.} We use a logistic regression model for classification. For the reductions approach, we use the implementation in the \textsf{fairlearn} package \footnote{\url{https://github.com/fairlearn/fairlearn}}. We set $T=50$, $\eta=2.0$ and $B=100$ for all experiments. We split the data into 75\% for training and 25\% for testing. We repeat the splitting over 10 trials.

\textbf{Effect of privacy.} We plot in Figure \ref{fig:exp_adult} the resulting discrimination violation and model error against increasing privacy levels $\epsilon$ for the predictor $\hat{Y}$ resulting from step 1 , trained on all the training data, and the two-step predictor $\tilde{Y}$ trained on $S_1$ and $S_2$ (split half and half). We observe that $\tilde{Y}$ achieves lower discrimination than $\hat{Y}$ across the different privacy levels. This comes at a cost of lower accuracy, which improves at lower privacy regimes (large epsilon). The predictor of step $1$ only begins to suffer on error when the privacy level is low enough as the fairness constraint is void at high levels of privacy (small epsilon).

Code to reproduce Figure \ref{fig:exp_adult} is made publicly available \footnote{\url{https://github.com/husseinmozannar/fairlearn_private_data}}.

\section{Discussion and Extensions} \label{sec:missing}

Could this approach for private demographic data be used to learn non-discriminatory predictors under other forms of deficiency in demographic information? In this section, we consider another case of interest: when individuals retain the choice of whether to release their sensitive information or not, as in the example of credit card companies. Practically, this means that the learner's data contains one part that has protected attribute labels and another that doesn't.

\textbf{Privacy as effective number of labeled samples.} As a first step towards understanding this setting,
suppose we are given $n_\ell$ fully labeled samples:
$S_\ell = \{(x_1,a_1,y_1),\cdots,(x_{n_\ell},a_{n_\ell},y_{n_\ell}) \}$ drawn  $i.i.d$ from an unknown distribution $\mathbb{P}$ over $\mathcal{X} \times \mathcal{A} \times \mathcal{Y}$ where $\mathcal{A}=\{0,1,\cdots,|\mathcal{A}|-1\}$ and $\mathcal{Y}=\{0,1,\cdots,|\mathcal{Y}|-1\}$, and $n_u$ samples that are missing the protected attribute:
$S_u = \{(x_1,y_1),\cdots,(x_{n_u},y_{n_u}) \}$ drawn $i.i.d$ from the marginal of $\mathcal{P}$ over $\mathcal{X}\times \mathcal{Y}$.
Define $n:=n_\ell + n_u$, $S= S_\ell \cup S_u$ and let $\beta >0$  be such that $n_\ell := \beta n$  and $n_u= (1-\beta) n$. 
 This data assumption is equivalent to having individuals not reporting their attributes uniformly at random with probability $1-\beta$.
The objective is to learn a non-discriminatory predictor $\hat{Y}$ from the data $S$.

To mimic step 1 of our methodology, we propose to modify the reductions approach, so as to allow the learner, Q-player, to learn on the entirety of $S$ while the auditor, $\boldsymbol{\lambda}$-player, uses only $S_\ell$. We do this by first defining a two data set version of the Lagrangian, as such:
\begin{equation}
L^{S,S_\ell}(Q,\boldsymbol\lambda) = \err^S(Q) + \boldsymbol\lambda ^\top (M\boldsymbol\gamma^{S_\ell}(Q) - \alpha \mathbf{1}). \label{eq:Lagrangian_missing_step1}
\end{equation}
This changes Algorithm \ref{alg:reductions} in two key ways: first, the update of $\boldsymbol{\theta}$ now only relies on $S_\ell$ and, second, the best response of the learner is still a cost-sensitive learning problem, however now the cost depends on whether sample $i$ is in $S_\ell$ or $S_u$. If it is in $S_u$, i.e. it does not have a group label, then the instance loss is the misclassification loss, while if it is in $S_\ell$ its loss is defined as before. Lemma \ref{lemma:step1_guarantees_missing} characterizes the performance of the learned predictor $\hat{Y}$ using the approach just described.

\begin{lemma}
    Given a hypothesis class $\mathcal{H}$, a distribution over $(X,A,Y)$, $B \in \mathbb{R}^+$ and any $\delta \in (0,1/2)$, then with probability greater than $1- \delta$, if  $n_\ell \geq \frac{8\log{4 |\A|/\delta}}{\min_{ya} \P_{ya}}$, $\alpha_n = 2 \sqrt{\frac{\log{32|\A|/\delta}}{n_\ell \min_{ya}\P_{ya}}}$ and we let $\nu  = \mathfrak{R}_{n}(\mathcal{H}) + \sqrt{\frac{\log{4/\delta}}{n}}$, then running the modified Algorithm 1 on data set $S$ and $S_l$    with $T \geq \frac{16 \log(4 |\mathcal{A}|+1)}{\nu^2}$ and learning rate $\eta = \frac{\nu}{8 B}$ returns a predictor $\hat{Y}$ satisfying the following:
	\begin{flalign*}
	&\err(\hat{Y}) \leq_{\delta} \err(Y^*)  + 4 \mathfrak{R}_{n}(\mathcal{H}) + 4 \sqrt{\frac{\log{4/\delta}}{n}} \\
	&\disc(\hat{Y}) \leq_{\delta}   \frac{2}{B}  + 6 \mathfrak{R}_{\frac{\min_{ya}n_\ell \P_{ya}}{2}}(\mathcal{H})  + 10 \sqrt{\frac{2\log{32 |\A|/\delta}}{n_\ell \min_{ya}\P_{ya}}} \ 
	\end{flalign*}
	
	\label{lemma:step1_guarantees_missing}
\end{lemma}
A short proof of Lemma \ref{lemma:step1_guarantees_missing} can be found in Appendix \ref{apx:proofs}. Notice the similarities between Lemma \ref{lemma:step1_guarantees} and \ref{lemma:step1_guarantees_missing}. The error bound we obtain depends on the entire number of samples $n$ as in the privacy case and the discrimination guarantee is forcibly controlled by the number of labeled group samples $n_\ell$. We can thus interpret the discrimination bound in Lemma \ref{lemma:step1_guarantees} as having an effective number of samples controlled by the privacy level $\epsilon$.

\paragraph{Individual choice of reporting} It is more reasonable to assume that individuals choice to report their protected attributes may depend on their own characteristics, let $t(x,y,a) \in (0,1]$ (reporting probability function) be the probability that an individual $(x,y,a)$ chooses to report their protected attributes; the setting of Lemma \ref{lemma:step1_guarantees_missing} is equivalent to a choice of  $t(x,y,a)=c$ for some $c \in (0,1]$. Starting from a dataset of individuals $S$ of size $n$ sampled  $i.i.d$ from  $\mathbb{P}^n$, each individual $i$ flips a coin with bias $t(x_i,y_i,a_i)$ and accordingly chooses to include their attribute $a_i$ in $S$. The result of this process is a splitting of $S$ into 
$S_\ell = \{(x_1,a_1,y_1),\cdots,(x_{n_\ell},a_{n_\ell},y_{n_\ell}) \}$ (individuals who report their attributes) and
$S_u = \{(x_1,y_1),\cdots,(x_{n_u},y_{n_u}) \}$ (individuals who do not report). The goal again is to learn a non discriminatory predictor $\hat{Y}$.

The immediate question is if we can use our modified algorithm with the two-dataset Langragian \eqref{eq:Lagrangian_missing_step1} and obtain similar guarantees to those in Lemma \ref{lemma:step1_guarantees_missing} in this more general setting. This question boils down to asking if the naive empirical estimate of discrimination is consistent and the answer depends both on the reporting probability function $t$ and the notion of discrimination considered as illustrated in the following proposition.

\begin{proposition}
Consider $( \mathcal{E}_1, \mathcal{E}_2)$-non-discrimination with respect to $A$ and fix a reporting probability function $t: \mathcal{X} \times \mathcal{Y} \times \mathcal{A} \to (0,1]$. Define the random variable $T: \mathcal{X} \times \mathcal{Y} \times \mathcal{A} \to \{0,1\} $ where $\mathbb{P}(T(x,y,a)=1)=t(x,y,a)$, then if $T$ and $\mathcal{E}_1$ are  conditionally independent given $\{A,\mathcal{E}_2\}$, we have as  $n\to \infty$ for all $a\in\mathcal{A}$
	\begin{equation*}
	\bP^{S_\ell}\left( \mathcal{E}_1| \mathcal{E}_2,A=a\right) \to_p  \bP\left( \mathcal{E}_1| \mathcal{E}_2,A=a\right)
	\end{equation*}
where $S_\ell$ is generated for each $n$ via the process described previously.
\label{prop:missing_at_choice}
\end{proposition}
\begin{proof}
Given $a\in\mathcal{A}$, our estimate 	$\bP^{S_\ell}\left( \mathcal{E}_1| \mathcal{E}_2,A=a\right)$ is nothing but the  empirical estimator of $\bP\left( \mathcal{E}_1| \mathcal{E}_2,A=a,T=1\right)$ where $\{T=1\}$ denotes the event that an individual does report their attributes and are thus included in $S_\ell$. As an immediate consequence we have:
	\begin{equation*}
	\bP^{S_\ell}\left( \mathcal{E}_1| \mathcal{E}_2,A=a\right) \to_p  \bP\left( \mathcal{E}_1| \mathcal{E}_2,A=a,T=1\right)
	\end{equation*}
Now as required by the statement of the proposition,  $T$ and $\mathcal{E}_1$ are  conditionally independent given $\{A,\mathcal{E}_2\}$ :
\begin{equation*}
\bP\left( \mathcal{E}_1| \mathcal{E}_2,A=a,T=1\right) = \bP\left( \mathcal{E}_1| \mathcal{E}_2,A=a\right)
\end{equation*}
which completes the proof of the proposition. Note that the event $\{\mathcal{E}_1| \mathcal{E}_2,A=a,T=1\}$ has strictly positive probability as the reporting probability function is strictly positive.
\end{proof}

If the independence condition in  Proposition \ref{prop:missing_at_choice} is satisfied, then the immediate consequence is that we can run Algorithm \eqref{alg:reductions} and obtain learning guarantees.\\
To make things concrete, suppose our notion of non-discrimination is EO, consider any reporting probability function  of the form $t_1: \mathcal{Y} \times \mathcal{A} \to (0,1]$ (does not depend on the non-sensitive attributes) and suppose our hypothesis class consists of functions that depend only on $X$. The conditional independence condition in Proposition \ref{prop:missing_at_choice} thus holds and  we can estimate the discrimination of any predictor in our class using $S_\ell$. The only change to Lemma \ref{lemma:step1_guarantees_missing} in this setup is that the effective number of samples in the discrimination bound is now: $n \min_{ya} \mathbf{P}_{ya} \cdot \mathbf{T}_{ya}$ where $\mathbf{T}_{ya}= \bP(T=1|Y=y,A=a)$ ($T$ is the r.v. that denotes reporting); the proof of this observation is immediate.

\paragraph{Trade-offs and proxies} 
To complete the parallel with the proposed methodology, what remains is to mimic step 2, to devise ways to have lower sample complexities to achieve non-discrimination. Clearly the dependence on $n_\ell$ in Lemma \ref{lemma:step1_guarantees_missing} is statistically necessary without any assumptions on the data generating process and the only area of improvement is to remove the dependence on the complexity of the model class. If the sensitive attribute is never available at test time, we cannot apply the post-processing procedure of \cite{hardt2016equality} in a two-stage fashion \cite{woodworth2017learning}.

In practice, to compensate for the missing direct information, if legally permitted, the learner may leverage multiple sources of data and combine them to obtain indirect access to the sensitive information \cite{kallus2019assessing} of individuals. The way this is modeled mathematically is by having recourse to proxies. One of the most widely used proxies is the 
Bayesian Improved Surname Geocoding (BISG) method, BISG is used to estimate race membership given the last name and geolocation of an individual \cite{adjaye2014using,fiscella2006use}. Using this proxy, one can impute the missing membership labels and then proceed to audit or learn a predictor. But a big issue with proxies is that they may lead to biased estimators for discrimination \cite{kallus2019assessing}. In order to avoid these pitfalls, one promising line of investigation is to learn it simultaneously with the predictor. 

What form of proxies can help us measure the discrimination of a certain predictor $\hat{Y}: \mathcal{X} \to \mathcal{Y}$? Some of the aforementioned issues are due to the fact that features $X$ are in general insufficient to estimate group membership, even through the complete probabilistic proxy $\bP(A|X)$. In particular for EO, if $A$ is not completely identifiable from $X$ then using this proxy leads to inconsistent estimates. In contrast, if we have access to the probabilistic proxy $\bP(A|X,Y)$, we then propose the following estimator (see also \cite{chen2019fairness})
\begin{equation}
\tilde{\gamma}^{S}_{ya}(\hat{Y})=\frac{\sum_{i=1}^n\hat{Y}(x_i)\mathbf{1}(y_i = y) \bP(A = a |x_i,y_i)}{ \sum_{i=1}^n \mathbf{1}(y_i = y)\bP(A = a |x_i,y_i)},
\end{equation}
which enjoys consistency, via a relatively straightforward proof found in Appendix \ref{apx:proofs}.
\begin{lemma}\label{lemma:unbiasedestimator}
	Let $S=\{(x_i,a_i,y_i)\}_{i=1}^{n}$ i.i.d. $\sim$ $\bP^n(A,X,Y)$,
	the estimator $\tilde{\gamma}^{S}_{ya}$ is consistent. As $n\to \infty$
	\begin{equation*}
	\tilde{\gamma}^{S}_{ya} \to_p  \gamma_{ya}.
	\end{equation*}
\end{lemma}
We end our discussion here by pointing out that if such a proxy can be efficiently learned from samples, then it can reduce a missing attribute problem effectively to a private attribute problem, allowing us to use much of the same machinery presented in this paper.

\section{Conclusion} \label{sec:conclusion}
We studied learning non-discriminatory predictors when the protected attributes are
privatized or noisy. We observed that, in the population limit, non-discrimination
against noisy attributes is equivalent to that against original attributes. We showed this
to hold for various fairness criteria. We then characterized the amount of difficulty,
in sample complexity, that privacy adds to testing non-discrimination. Using
this relationship, we proposed how to carefully adapt existing non-discriminatory
learners to work with privatized protected attributes. Care is crucial, as naively
using these learners may create the illusion of non-discrimination, while continuing
to be highly discriminatory. We ended by highlighting future work on how to learn predictors in the absence of any demographic data or prior proxy information.

\bibliographystyle{alpha}
\bibliography{ref}
\newpage

\appendix

\section{Deferred Proofs}\label{apx:proofs}

Two important notation we use throughout are: for empirical versions of quantities based on data set $S$ we use a  superscript $S$ and the "probabilistic" inequality $a\le_{\delta}b$ signifies that $a$ is less than $b$ with probability greater than $1-\delta$.
\subsection{Section \ref{sec:auditing}}

The below example illustrates that non-discrimination with respect to $A$ and $Z$ are not equivalent for general predictors.

\begin{example}
Let $|\mathcal{A}|=2$, consider the  predictors $\hat{Y}_1=h(X,Z)$ and $\hat{Y}_2=h(X,Z)$  with the  conditional probabilities for $y \in \{0,1 \}$ defined in table \ref{table:example_nondisc} with the function  $h(x)= \begin{cases}
0 \ \textrm{if} \ x\leq 1/2 \\
\frac{1}{2x} \ \textrm{if} \ x>1/2 \\
\end{cases}
$, note that $h(x) \in[0,1]$ so that the predictor $\hat{Y}_2$ is valid.

\begin{table*}[ht]
\begin{center}
\begingroup
\setlength{\tabcolsep}{10pt} 
\renewcommand{\arraystretch}{1.5} 
\begin{tabular}{|c|c|c|}
	\hline 
(a,z)	&  \small{$\bP(\hat{Y}_1=1|A=a,Z=z, Y=y)$}  &  \small{$\bP(\hat{Y}_2=1|A=a,Z=z, Y=y)$}  \\ 
	\hline 
(0,0)	& $\frac{1}{2\pi}$	& $h( \bP(A=0|Z=0,Y=y))$ 	\\ 
 \hline 
(0,1)	&	0	 & $h( \bP(A=0|Z=1,Y=y))$   \\ 
	\hline 
(1,0)	&  $0$	& $h( \bP(A=1|Z=0,Y=y))$	\\ 
	\hline 
(1,1)	&  	$\frac{1}{2\pi}$	& $h( \bP(A=1|Z=1,Y=y))$	\\ 
	\hline 
\end{tabular} 
\endgroup
\end{center}
\caption{Predictors used to show non-equivalence of discrimination with respect to $A$ and $Z$ when predictors are a function of $Z$.}
\label{table:example_nondisc}
\end{table*}

The predictors $\hY_1$ and $\hY_2$ are designed by construction to show that non discrimination with respect to $A$ and $Z$ are not statistically equivalent. We show that $\hY_1$ satisfies EO with respect to $A$ but violates it with respect to $Z$ and $\hY_2$ is non-discriminatory with respect to $Z$ but is for $A$.

\begin{proof}
For $\hY_1$: first we show it satisfies EO for A:
\begin{flalign*}
    &\bP(\hY_1=1|A=a,Y=y) &\\&= \pi \bP(\hY_1=1|A=a,Z=a,Y=y) \\&+ (1-\pi) \bP(\hY_1=1|A=a,Z=\bar{a},Y=y) = \frac{1}{2}
\end{flalign*}
Since the above is no different for $a,y\in \{0,1 \}$, $\hY_1$  satisfies EO. Now with respect to Z:
\begin{flalign*}
    &\bP(\hY_1=1|Z=a,Y=y) &\\&= \bP(A=a|Z=a,Y=y) \bP(\hY_1=1|Z=a,A=a,Y=y) \\&+ \bP(A=\bar{a}|Z=a,Y=y) \bP(\hY_1=1|Z=a,A=\bar{a},Y=y) \\
    &= \frac{\bP(A=a|Z=a,Y=y)}{ 2 \pi}
\end{flalign*}
Therefore if and only if $\bP(A=0|Z=0,Y=y) = \bP(A=1|Z=1,Y=y) $ is it also non discriminatory with respect to $Z$.

For $\hY_2$: by construction only one of  $(\bP(\hat{Y}_2=1|A=a,Z=a, Y=y),$ $\bP(\hat{Y}_2=1|A=\bar{a},Z=a, Y=y))$ is non-zero as only one of $(\bP(A=1|Z=a,Y=y),\bP(A=0,Z=a,Y=y))$ is greater than $1/2$ and so:
\begin{flalign*}
    &\bP(\hY_2=1|Z=a,Y=y)&\\ &= \bP(A=a|Z=a,Y=y) \bP(\hY_2=1|Z=a,A=a,Y=y) \\&+ \bP(A=\bar{a}|Z=a,Y=y) \bP(\hY_2=1|Z=a,A=\bar{a},Y=y) \\
&= \frac{1}{2} 
\end{flalign*}
Therefore $\hY_2$ satisfies EO with respect to $Z$, on the other side:
\begin{flalign*}
    &\bP(\hY_2=1|A=a,Y=y) &\\&= \bP(Z=a|A=a,Y=y) \bP(\hY_2=1|Z=a,A=a,Y=y) \\&+ \bP(Z=\bar{a}|A=a,Y=y) \bP(\hY_2=1|Z=\bar{a},A=a,Y=y) \\
&=  \pi \cdot h( \bP(A=a|Z=a,Y=y)) \\&+ (1-\pi) \cdot h( \bP(A=a|Z=\bar a,Y=y))
\end{flalign*}
and is discriminatory with respect to $A$ unless $\bP(A=a,Y=y)=\bP(A=\bar{a},Y=y)$ for $y \in \{ 0,1\}$ as $\bP(A=a|Z=a,Y=y) = \frac{\pi \bP(A=a,Y=y)}{\bP(Z=a,Y=y)}$.
\end{proof}
\end{example}

\textit{ \noindent \textbf{Proposition \ref{prop:forzeqfora}}
	Consider any exact non-discrimination notion among equalized odds, demographic parity, accuracy parity, or equality of false discovery/omission rates. Let $\hat{Y}:=h(X)$ be a binary predictor, then $\hat{Y}$ is non-discriminatory with respect to $A$ if and only if it is non-discriminatory with respect to $Z$. 
}
\begin{proof}
	The proof of the above proposition relies on the fact that if $\hat{Y}$ is independent of $Z$ given $A$, then the conditional probabilities with respect to $Z$ and $A$ are related via a linear system.
	
	We prove the proposition by considering a general formulation of the constraints we previously mentioned, let $ \mathcal{E}_1, \mathcal{E}_2$ be two probability events defined with respect to $(X,Y,\hat{Y})$, then consider the following probability:
	
	\begin{flalign}
	 &\bP\left( \mathcal{E}_1| \mathcal{E}_2,Z=a\right) \nonumber \\&= \sum_{a' \in \A}  \bP\left( \mathcal{E}_1| \mathcal{E}_2,Z=a,A=a')\right) \bP(A=a'|  \mathcal{E}_2,Z=z) \nonumber&\\
	 &\overset{(a)}{=} \sum_{a' \in \A}  \bP\left( \mathcal{E}_1| \mathcal{E}_2,A=a')\right) \bP(A=a'|  \mathcal{E}_2,Z=z)\nonumber\\
	 	 &\overset{}{=} \sum_{a' \in \A}  \bP\left( \mathcal{E}_1| \mathcal{E}_2,A=a')\right) \frac{\bP(Z=z|A=a',\mathcal{E}_2)\bP(A=a',\mathcal{E}_2)}{\bP(Z=a',\mathcal{E}_2)}  \nonumber \\
	 	  &\overset{}{=} \bP\left( \mathcal{E}_1| \mathcal{E}_2,A=a)\right) \frac{\pi\bP(A=a,\mathcal{E}_2)}{\bP(Z=a,\mathcal{E}_2)} \nonumber \\&+  \sum_{a' \in \A \setminus \{a\}}  \bP\left( \mathcal{E}_1| \mathcal{E}_2,A=a')\right) \frac{\bar\pi\bP(A=a',\mathcal{E}_2)}{\bP(Z=a',\mathcal{E}_2)}   \label{prop1eq:lineareq}
	\end{flalign}
	
	step $(a)$ follows as $(X,Y,\hat{Y})$ are independent of $Z$ given $A$. We define non-discrimination with respect to $A$ as having (similarly defined with respect to $Z$):
	\[
	\bP\left( \mathcal{E}_1| \mathcal{E}_2,A=a\right) = \bP\left( \mathcal{E}_1| \mathcal{E}_2,A=a'\right) \quad \forall a,a' \in \A
	\]
	
	Assume first that the predictor $\hat{Y}$ is non-discriminatory with respect to $A$, hence 
	$\exists c$ where $\forall a \in \mathcal{A}$ we have $ \bP( \mathcal{E}_1| \mathcal{E}_2,A=a)=c$, hence  by \eqref{prop1eq:lineareq} for all $a \in \A$:
	\begin{flalign*}
	 &\bP\left( \mathcal{E}_1| \mathcal{E}_2,Z=a\right) 
	 	  &\\&\overset{}{=} c \frac{\pi\bP(A=a,\mathcal{E}_2)}{\bP(Z=a,\mathcal{E}_2)} +  \sum_{a' \in \A \setminus \{a\}} c \frac{\bar\pi\bP(A=a',\mathcal{E}_2)}{\bP(Z=a',\mathcal{E}_2)} = c
	\end{flalign*}	
	which proves that $\hat{Y}$ is also non-discriminatory with respect to $A$. 
	
	Now, assume instead that the predictor $\hat{Y}$ is non-discriminatory with respect to $Z$, hence 
	$\exists c$ where $\forall a \in \mathcal{A}$ we have $ \bP( \mathcal{E}_1| \mathcal{E}_2,Z=a)=c$. Let $P$ be the following $|\A| \times |\A|$ matrix:
		\[
	\begin{cases}
	P_{i,i} =\frac{\pi\bP(A=i,\mathcal{E}_2)}{\bP(Z=i,\mathcal{E}_2)}   \ \text{for} \ i\in \mathcal{A}\\
	P_{i,j} = \frac{\bar\pi\bP(A=i,\mathcal{E}_2)}{\bP(Z=j,\mathcal{E}_2)} \ \text{for} \ i,j\in \mathcal{A} \ \text{s.t.} i \neq j\\
	\end{cases}
	\]
	Then we have the following linear system of equations:
\begin{align*}
	&\begin{bmatrix} 
		\bP( \mathcal{E}_1| \mathcal{E}_2,Z=0) \\
		\vdots \\
	\bP( \mathcal{E}_1| \mathcal{E}_2,Z=|\A|-1)
	\end{bmatrix}
	= P 	\begin{bmatrix} 
		\bP( \mathcal{E}_1| \mathcal{E}_2,A=0) \\
		\vdots \\
	\bP( \mathcal{E}_1| \mathcal{E}_2,A=|\A|-1)
	\end{bmatrix} \\
	& \textrm{denoted by} \ \mathbf{z} = P \mathbf{a}
	\end{align*}
In our case $\mathbf{a}=c \cdot \mathbf{1}$, and we show that also $\mathbf{z}=c \cdot \mathbf{1}$. Let us state some properties of the matrix $P$:
	\begin{itemize}
		\item $P$ is row-stochastic
		\item $P$ is invertible (we later show the exact form of this inverse implying its existence, however its existence is easy to see as all rows are linearly independent as $\pi \neq \bar\pi$ and $\forall a$, 	$\bP( Z=a ,\mathcal{E}_2)>0$ ).
		\item As $P$ is row-stochastic and invertible, the rows of $P^{-1}$ sum to 1, this is as $P \mathbf{1} = \mathbf{1} \iff \mathbf{1} = P^{-1} \mathbf{1}$ 
	\end{itemize}
		By the second property $z = c \cdot P^{-1} \mathbf{1}$ and by the third property we have $P^{-1} \mathbf{1} = \mathbf{1}$ which in turn means that $z = c \cdot \mathbf{1}$ and implies that $\hat{Y}$ is non-discriminatory with respect to $Z$.

As an extension, consider fairness notions  formulated as:
	\[
	\bP\left( \mathcal{E}_1,A=a| \mathcal{E}_2\right) = \bP\left( \mathcal{E}_1,A=a'| \mathcal{E}_2\right) \quad \forall a,a' \in \A
	\]
	Then we have 
	\begin{flalign*}
	 &\bP\left( \mathcal{E}_1,Z=a| \mathcal{E}_2\right)\\ &= \sum_{a' \in \A}  \bP\left( \mathcal{E}_1,Z=a| \mathcal{E}_2,A=a'\right) \bP(A=a'|  \mathcal{E}_2) \nonumber&\\
	 &= \sum_{a' \in \A}  \bP\left( \mathcal{E}_1| \mathcal{E}_2,A=a'\right) \bP\left( Z=a| A=a')\right) \bP(A=a'|  \mathcal{E}_2) \\
	 &= \sum_{a' \in \A}   \bP\left( \mathcal{E}_1,A=a'| \mathcal{E}_2\right) \bP\left( Z=a| A=a')\right) \\
	 &= \pi \bP\left( \mathcal{E}_1,A=a'| \mathcal{E}_2\right)\sum_{a' \in \A \setminus \{a \}}  \bar{\pi} \bP\left( \mathcal{E}_1,A=a'| \mathcal{E}_2\right) 
	\end{flalign*}
By the same arguments as above, for these notions of fairness  $\hat{Y}$ is non-discriminatory with respect to $A$ if and only if it is non-discriminatory with respect to $Z$.
	
	For concreteness, we derive equation $\eqref{prop1eq:lineareq}$ for each of the fairness notions we mentioned.
	First a detailed derivation for equalized odds, we let $\mathcal{E}_1= \{\hY=1\}$ and for EO we need to apply the above reasoning for $|\mathcal{Y}|$ events $\mathcal{E}_{2_y}=\{Y=y\}$:
	
	\begin{flalign*}
	\nonumber	&\bP(\hat{Y}=1 |Y=y, Z =a) \\&= \sum_{a'} \bP(\hat{Y}=1 |Y=y, Z =a,A =a') \bP(A=a'|Z=a, Y=y) &\\
	\nonumber	&\overset{(a)}{=} \sum_{a'} \bP(\hat{Y}=1 |Y=y,A =a') \bP(A=a'|Z=a, Y=y) \\
	\nonumber	&= \sum_{a'} \bP(\hat{Y}=1 |Y=y,A =a')  \frac{\bP(Z=a,Y=y|A=a')\bP(A=a')}{\bP(Z=a,Y=y) }\\
	\nonumber	&\overset{(b)}{=}  \sum_{a'} \bP(\hat{Y}=1 |Y=y,A =a')  \frac{\bP(Z=a|A=a')\bP(Y=y|A=a')\bP(A=a')}{\bP(Z=a,Y=y) }\\
		&= \bP(\hat{Y}=1 |Y=y,A =a) \frac{\pi \P_{ya}}{ \pi \P_{ya} + \sum_{a'' \setminus a }\bar{\pi}\P_{ya''} } \nonumber \\
		&+ \sum_{a'\setminus a} \bP(\hat{Y}=1 |Y=y,A =a') \frac{\bar{\pi} \P_{ya'}}{\pi \P_{ya}+ \sum_{a'' \setminus a }\bar{\pi} \P_{ya''}}
	\end{flalign*}
	First line by conditioning on $A$ and then taking expectation, $(a)$ is by our assumption of the conditional independence of $Z,\hY$ given $A$ and step $(b)$ by the independence of $Z$ and $Y$ given $A$.

	Similarly for demographic parity with denoting $p_a = \bP(A=a)$:
	\begin{flalign*}
		&\bP(\hat{Y}=1 | Z =a) &
		\\&= \sum_{a'} \bP(\hat{Y}=1 | Z =a,A =a') \bP(A=a'|Z=a) \\
		&= \sum_{a'} \bP(\hat{Y}=1 |A =a') \bP(A=a'|Z=a) \\
		&= \sum_{a'} \bP(\hat{Y}=1 |A =a') \frac{\bP(Z=a|A=a')p_{a'}}{\sum_{a''}\bP(Z=a|A=a'') p_{a''} }\\
		&=  \sum_{a'} \bP(\hat{Y}=1 |A =a') \frac{\bP(Z=a|A=a') p_{a'}}{\sum_{a''}\bP(Z=a|A=a'')  p_{a''} }\\
		&= \bP(\hat{Y}=1 |A =a) \frac{\pi  p_{a}}{ \pi  p_{a} + \sum_{a'' \setminus a }\bar{\pi}  p_{a''} } + \sum_{a'\setminus a} \bP(\hat{Y}=1 |A =a') \frac{\bar{\pi}  p_{a'}}{\pi  p_{a} + \sum_{a'' \setminus a }\bar{\pi}  p_{a''}} 
	\end{flalign*}
	
	Now for equal accuracy among groups:
	\begin{flalign*}
		&\bP(\hat{Y}\neq Y | Z =a) &
		\\&= \sum_{a'} \bP(\hat{Y}\neq Y  | Z =a,A =a') \bP(A=a'|Z=a) \\
		&=\bP(\hat{Y}\neq Y  |A =a') \frac{\pi  p_{a}}{ \pi  p_{a} + \sum_{a'' \setminus a }\bar{\pi}  p_{a''} } + \sum_{a'\setminus a}\bP(\hat{Y}\neq Y  |A =a') \frac{\bar{\pi}  p_{a'}}{\pi  p_{a} + \sum_{a'' \setminus a }\bar{\pi}  p_{a''}} 
	\end{flalign*}
	
	And finally for equality of false discovery/omission rates, denote $p_{\hat{y},a} := \bP(\hat{Y}=\hat{y},A=a)$:
	
	\begin{flalign*}
		&\bP(\hat{Y}\neq Y | \hat{Y} =\hat{y}, Z =a) &
		\\&= \sum_{a'} \bP(\hat{Y}\neq Y  |\hat{Y} =\hat{y}, Z =a,A =a') \bP(A=a'|Z=a,  \hat{Y} =\hat{y} ) \\
		&=\sum_{a'} \bP(\hat{Y}\neq Y  |\hat{Y} =\hat{y},A =a') \bP(A=a'|Z=a,  \hat{Y} =\hat{y} ) \\ 
		&= \sum_{a'} \bP(\hat{Y}\neq Y | \hat{Y} =\hat{y}, A =a')  \frac{\bP(Z=a,\hat{Y}=\hat{y}|A=a' )p_{a'}}{\sum_{a''}\bP(Z=a, \hat{Y}=\hat{y}|A=a'') p_{a''} }\\
		&= \bP(\hat{Y}\neq Y | \hat{Y} =\hat{y}, A =a')  \frac{\pi  p_{\hat{y},a}}{ \pi p_{\hat{y},a} + \sum_{a'' \setminus a }\bar{\pi} p_{\hat{y},a''} } + \sum_{a'\setminus a}  \bP(\hat{Y}\neq Y | \hat{Y} =\hat{y}, A =a')\frac{\bar{\pi} p_{\hat{y},a'}}{\pi p_{\hat{y},a} + \sum_{a'' \setminus a } \bar{\pi} p_{\hat{y},a''}} 
	\end{flalign*}
	Note that we did not need the independence of $\hat{Y}$ and $Z$ given $A$ to express $\bP(\hat{Y}\neq Y | \hat{Y} =\hat{y}, Z =a) $ in terms of $\bP(\hat{Y}\neq Y  |\hat{Y} =\hat{y},A =a)$ so that the equivalence follows without our assumption for equality of FDR. However, to be able to do the inversion of statistics we require the assumption.
\end{proof}

The version of the below Lemma that appears in the text is obtained by plugging in $\pi = \frac{e^\epsilon}{|\A| -1 + e^{\epsilon}}$.

\textit{\noindent \paragraph{Lemma \ref{lemma:concentration_of_estimator} } For any $\delta \in (0,1/2)$, any binary predictor $\hat{Y}:=h(X)$, denote by $\P_{ya}:= \bP(Y=y,A=a)$, $\Gamma_{ya}:= \abs{q_{y,a}(\hat{Y}) - \gamma_{y,0}(\hat{Y})}$ and $\tilde{\Gamma}_{ya}^S$ our proposed estimator based on $S$, let $C= \frac{\pi + |\A|-1}{|\A|\pi-1}$, then if
$n \geq \frac{8\log(|8\A|/\delta)}{\min_{ya}\P_{ya}}$, we have:
 \begin{equation*}
     \bP\left( \max_{ya}| \tilde{\Gamma}_{ya}^S - \Gamma_{ya}| >
     \ \sqrt{\frac{\log(16/\delta)}{2n}} \frac{4|\A|C^2}{\min_{ya}\P_{ya}^2}
     \right) \leq \delta
 \end{equation*}
}

\begin{proof}
 \textbf{Step 1:} \textit{Deriving our estimator.}

The following equality allows to invert the statistics of the population with respect to $Z$ that we have sample estimates of to get population estimates of the true statistics with respect to $A$. We write
	\begin{flalign}
	\label{eq:phaty|y,zrelation}	&\bP(\hat{Y}=1 |Y=y, Z =a)= \\& \sum_{a'} \bP(\hat{Y}=1 |Y=y, Z =a,A =a') \bP(A=a'|Z=a, Y=y) \nonumber&\\
	\nonumber	&\overset{(a)}{=} \sum_{a'} \bP(\hat{Y}=1 |Y=y,A =a') \bP(A=a'|Z=a, Y=y) \\
	\nonumber	&= \sum_{a'} \bP(\hat{Y}=1 |Y=y,A =a')  \nonumber \frac{\bP(Z=a,Y=y|A=a')\bP(A=a')}{\bP(Z=a,Y=y) }\\
	\nonumber	&\overset{(b)}{=}  \sum_{a'} \bP(\hat{Y}=1 |Y=y,A =a') \nonumber \frac{\bP(Z=a|A=a')\bP(Y=y|A=a')\bP(A=a')}{\bP(Z=a,Y=y) } \nonumber\\
		&= \pi \bP(\hat{Y}=1 |Y=y,A =a) \frac{\bP(Y=y,A=a)}{\bP(Z=a,Y=y)} 
	+ \sum_{a'\setminus a} \bar{\pi} \bP(\hat{Y}=1 |Y=y,A =a') \frac{\bP(Y=y,A=a')}{\bP(Z=a,Y=y) } \nonumber
	\end{flalign}
	First line is by conditioning on $A$ and then taking expectation, step $(a)$ is by our assumption of the conditional independence of $Z,\hY$ given $A$ and step $(b)$ by the independence of $Z$ and $Y$ given $A$.
	
Let $G$ be the $\mathcal{A} \times \mathcal{A}$ matrix be as such:
	$\begin{cases}
	G_{i,i} = \pi \frac{\bP(Y=y,A=i)}{\bP(Z=i,Y=y) }  \ \text{for} \ i\in \mathcal{A}\\
	G_{i,j} = \bar{\pi} \frac{\bP(Y=y,A=j)}{\bP(Z=i,Y=y) }  \ \text{for} \ i,j\in \mathcal{A} \ \text{s.t.} i \neq j\\
	\end{cases}$. Then we can write equation \eqref{eq:phaty|y,zrelation} as a linear system with $q_{ya}(\hat{Y})= \bP(\hat{Y}=1|Y=y,Z=a)$:
\begin{align*}
	\begin{bmatrix} 
		q_{y0} \\
		\vdots \\
		q_{y,|\mathcal{A}-1|}
	\end{bmatrix}
	&= G \begin{bmatrix} 
		\bP(\hY=1|Y=y,A=0) \\
		\vdots \\
		\bP(\hY=1|Y=y,A=|\mathcal{A}|-1)
	\end{bmatrix}\\
	 \  q_{y,.} &= G \ \bP\left(\hY=1|Y=y,A\right) \ \textit{(notation)}
\end{align*}
And thus by inverting G we can recover the population statistics. We show that the inverse of $G$ takes the following form:
\[
\begin{cases}
	G^{-1}_{i,i} = \frac{\pi + |\mathcal{A}|-2}{|\A|\pi -1    }  \frac{\bP(Z=i,Y=y) }{\bP(Y=y,A=i)}  \ \text{for} \ i\in \mathcal{A}\\
	G^{-1}_{i,j} = \frac{\pi-1}{|\A|\pi -1  } \frac{\bP(Z=j,Y=y) }{\bP(Y=y,A=i)}  \ \text{for} \ i,j\in \mathcal{A} \ \text{s.t.} i \neq j\\
	\end{cases}
\]
Let $i\neq j \in \mathcal{A}$:
\begin{flalign*}
&G_i G^{-1}_{,j} \\&= \sum_{k} G_{i,k} G^{-1}_{k,j} &\\
&=   \pi \frac{\bP(Y=y,A=i)}{\bP(Z=i,Y=y) } \cdot  \frac{ \pi -1}{|\A|\pi -1  } \frac{\bP(Z=j,Y=y) }{\bP(Y=y,A=i)} + \bar{\pi} \frac{\bP(Y=y,A=j)}{\bP(Z=i,Y=y) } \frac{\pi + |\mathcal{A}|-2 }{|\A|\pi -1    }  \frac{\bP(Z=j,Y=y) }{\bP(Y=y,A=j)} \\
&+ \sum_{k \setminus \{i,j\}}  \bar{\pi} \frac{\bP(Y=y,A=k)}{\bP(Z=i,Y=y) } \cdot   \frac{\pi -1 }{|\A|\pi -1    } \frac{\bP(Z=j,Y=y) }{\bP(Y=y,A=k)} \\
&= \frac{\bP(Z=j,Y=y)}{\bP(Z=i,Y=y)} \cdot  \frac{ \pi (\pi - 1)  +\frac{1- \pi}{|\A|-1} ( \pi + |\A| -2 +(|\A|-2) (\pi -1)) }{|\A|\pi -1    }\\
&=0
\end{flalign*}
And now for $i \in \mathcal{A}$
\begin{flalign*}
&G_i G^{-1}_{,i} \\&= \sum_{k} G_{i,k} G^{-1}_{k,i} &\\
&=   \pi \frac{\bP(Y=y,A=i)}{\bP(Z=i,Y=y) } \cdot   \frac{\pi + |\mathcal{A}|-2} {|\A|\pi -1     } \frac{\bP(Z=i,Y=y) }{\bP(Y=y,A=i)} + \sum_{k \setminus \{i\}}  \bar{\pi} \frac{\bP(Y=y,A=k)}{\bP(Z=i,Y=y) } \cdot   \frac{\pi - 1}{|\A|\pi -1     } \frac{\bP(Z=i,Y=y) }{\bP(Y=y,A=k)} \\
&=  \frac{ \pi ( \pi + |\A| -2) + \frac{1 - \pi}{|\A|-1} (\pi-1) (|\A|-1)}{|\A|\pi -1 }\\
&=1
\end{flalign*}
Which proves that it is indeed the inverse.

The matrix $G$ involves estimating the probabilities $\bP(Y=y,A=a)$ which we do not have access to but can similarly recover by noting that:
\begin{flalign}
  \nonumber  \Q_{yz} &= \sum_{a \in \mathcal{A}} \bP\left(Y=y,Z=z|A=a\right) \bP\left(A=a\right) &\\
       \nonumber  &= \sum_{a \in \mathcal{A}} \bP\left(Y=y|A=a\right) \bP\left(Z=z|A=a\right) \bP(A=a)\\
         &= {\pi} \bP\left(Y=y,A=z\right) + \sum_{a \neq z} \bar{\pi} \bP\left(Y=y,A=a\right) \label{eq:Qyz}
\end{flalign}

Let the matrix $\Pi \in \mathbb{R}^{|\mathcal{A}|\times|\mathcal{A}|}$ be as follows $\Pi_{i,j} = \pi$ if $i=j$ and $\Pi_{i,j}=\bar{\pi}$ if $i \neq j$. We know from equation \eqref{eq:Qyz} that:
\begin{align*}
	\begin{bmatrix} 
		\Q_{y0} \\
		\vdots \\
		\Q_{y,|\mathcal{A}|-1}
	\end{bmatrix}
	&= \Pi \begin{bmatrix} 
		\bP(Y=y,A=0) \\
		\vdots \\
		\bP(Y=y,A=|\mathcal{A}|-1)
	\end{bmatrix}\\
	 \  \Q_{y,.} &= \Pi \ \bP\left(Y=y,A\right) \ \textit{(notation)}
\end{align*}
Therefore $\Pi^{-1}_k \Q_{y,.}  = \bP(Y=y,A=k)$ where $\Pi^{-1}_k$ is the $k$'th row of $\Pi^{-1}$.
Now $\Pi^{-1}$ is as such: $\Pi^{-1}_{i,i}= \frac{\pi + |\mathcal{A}|-2 }{|\mathcal{A}|\pi -1}$  and $\Pi^{-1}_{i,j}= \frac{\pi -1}{|\mathcal{A}|\pi -1}$ if $i\neq j$  with the same proof as for the inverse of $G$.
Therefore our empirical estimator for $\bP(\hY=1|Y=y,A=a)$ is $\hat{G}_a^{-1} q_{y,.}^S$ where $\hat{G}^{-1}$ is defined with the empirical versions of the probabilities involved where $\bP(Y=y,A=a)$ is estimated by $\Pi^{-1}_a \Q_{y,.}^S$. One issue that arises here is that while the sum of our estimator entries sum to $1$, some entries might be in fact negative and therefore we need to project the derived estimator onto the simplex. We later discuss the implications of this step. 
 
 \textbf{Step 2:} \textit{Concentration of raw estimator.}
 
 Let us first denote some things:
$n^S_{y,z}=\sum_i\mathbf{1}(y_i=y,z_i=z)$, $\Q_{y,z}=\bP(Y=y,Z=z)$, and the random variables  $S_{y,z}=\{i:y_i=y,z_i=z\}$. \\
We have that $\b{E}[\hat{G}_z^{-1} q_{y,.}^S|S_{y,0},\cdots,S_{y,|\A|-1}]=G_z^{-1} q_{y,.}= \gamma_{y,z}$. Inspired by the proof of Lemma 2 in \cite{woodworth2017learning} we have:

\begin{flalign*}
&\nonumber\bP\left(|\hat{G}_z^{-1}q_{y,.}^S-\gamma_{yz}|>t\right)&\\&\overset{(a)}=\sum_{S_{y,0},\cdots,S_{y,|\A|-1}}\bP\left(|\hat{G}_z^{-1}q_{y,.}^S-\gamma_{yz}|>t|S_{y,0},\cdots,S_{y,|\A|-1}\right)\bP\left(S_{y,0},\cdots,S_{y,|\A|-1}\right) &\\
\nonumber &\overset{(b)}{\le} \bP\left(\cup_{a\in \mathcal{\A}} \{ n^S_{y,a}<\frac{n\Q_{y,a}} 2 \}\right) \\
&+\sum_{ \forall z, S_{yz}: n^S_{yz}\ge\frac{n\Q_{yz}}2}\bP\left(|\hat{G}_z^{-1}q_{y,.}^S-\gamma_{yz}|>t|S_{y,0},\cdots,S_{y,|\A|-1}\right)\bP\left(S_{y,0},\cdots,S_{y,|\A|-1}\right)&\\
\nonumber &\overset{(c)}\le  |A|  \exp{\left(- \frac{\min_{a}n\Q_{ya}}8\right)}\\&+\sum_{ \forall z, S_{yz}: n^S_{yz}\ge\frac{n\Q_{yz}}2}\bP\left(|\hat{G}_z^{-1}q_{y,.}^S-\gamma_{yz}|>t|S_{y,0},\cdots,S_{y,|\A|-1}\right)\bP\left(S_{y,0},\cdots,S_{y,|\A|-1}\right)
\end{flalign*}
Step $(a)$ follows by conditioning over over all $|\A|^n$ possible configurations of $S_{y,0},\cdots,S_{y,|\A|-1}\subset[n]$, step $(b)$ comes by splitting over  configurations where $\forall z, S_{yz}: n^S_{yz}\ge\frac{n\Q_{yz}}2$ and the complement of the previous event and upper bounding this complement by the probability that there $\exists z$ s.t.  $n^S_{yz}<\frac{n\Q_{yz}}2$. Finally step $(c)$ comes from a union bound and then a Chernoff bound on $n^S_{yz}\sim\text{Binomial}(n,\Q_{yz})$ and taking the minimum over $\Q_{ya}$. 
 
We now recall McDiarmid's inequality \cite{mcdiarmid1989method}. Let $W^n=(W_1,\cdots,W_n) \in \mathcal{W}^n$ be $n$ independent random variables and $f:\mathcal{W}^n \to \mathbb{R}$, if there exists constants $c_1,\cdots,c_n$ such that for all $i \in [n]$:
\[
\sup_{w_1,\cdots,w_i,w_i',\cdots,w_n} | f(w_1,\cdots,w_i,\cdots,w_n) -  f(w_1,\cdots,w_i',\cdots,w_n)| \leq c_i
\]
Then for all $\epsilon>0$:
\[
\bP\left(f(W_1,\cdots,W_n) - \bE [f(W_1,\cdots,W_n]\right) \leq 2 \exp \left( -\frac{2 \epsilon^2}{\sum_{i=1}^{n}c_{i}^2} \right)
\]

Now conditioned on $S_{y,0},\cdots,S_{y,|\A|-1}$, our estimator $\hat{G}_z^{-1}q_{y,.}^S$ is only a function of $\hat{Y}_1,\cdots,\hat{Y}_n$, we try to bound how much can our estimator change on two dataset $S$ and $S'$ differing by only one value of $\hat{Y}_i$:

For convenience denote by $C_1 = \frac{\pi + |\mathcal{A}|-2}{|\A|\pi -1    } $ and $C_2 = \frac{\pi -1}{|\A|\pi -1    } $:
\begin{flalign*}
 &\sup_{S,S'} \ |\hat{G}_z^{-1}q_{y,.}^S - \hat{G}_z^{-1}q_{y,.}^{S'}| &\\
 &= \left| C_1 \frac{\Pi^{-1}_z \Q^S_{y,.} }{\Q^S_{y,z}} q^S_{y,z}  + \sum_{a \in \A \setminus z} C_2 \frac{\Pi^{-1}_a \Q^S_{y,.} }{\Q^S_{y,z}} q^S_{y,a} - C_1 \frac{\Pi^{-1}_z \Q^S_{y,.} }{\Q^S_{y,z}} q^{S'}_{y,z}  - \sum_{a \setminus z} C_2 \frac{\Pi^{-1}_a \Q^S_{y,.} }{\Q^S_{y,z}} q^{S'}_{y,a} \right| \\
 &= \Bigg| C_1 \frac{\Pi^{-1}_z \Q^S_{y,.} }{\Q^S_{y,z}}\left( \frac{\sum_{i \in S} \hat{Y}_i \bI(Y_i=y,Z_i=z) }{n^S_{yz}}  - \frac{\sum_{i \in S'} \hat{Y}_i \bI(Y_i=y,Z_i=z) }{n^{S}_{yz}}\right)\\
 &+ \sum_{a \in \A \setminus z} C_2 \frac{\Pi^{-1}_a \Q^S_{y,.} }{\Q^S_{y,z}} \left( \frac{\sum_{i \in S} \hat{Y}_i \bI(Y_i=y,Z_i=a) }{n^S_{ya}}  - \frac{\sum_{i \in S'} \hat{Y}_i \bI(Y_i=y,Z_i=a) }{n^{S}_{ya}}\right)  \Bigg| \\
 & \leq \left|C_1  \frac{ \max_{a}\Pi^{-1}_a \Q^S_{y,.} }{\Q^S_{y,z}} \frac{1}{n^S_{yz}} \right|  = \left|C_1  \frac{ \max_{a} C_1 n^S_{ya} + C_2 (n-n^S_{ya})  }{n^S_{yz}} \frac{1}{n^S_{yz}} \right| \\
 & \leq \left| \left(\frac{C_1}{n^{S}_{yz}}\right)^2  n  \right| 
\end{flalign*}
Therefore by McDiarmid's inequality we have:
\begin{flalign*}
&\sum_{ \forall z, S_{yz}: n^S_{yz}\ge\frac{n\Q_{yz}}2}\bP\left(|\hat{G}_z^{-1}q_{y,.}^S-\gamma_{yz}|>t|S_{y,0},\cdots,S_{y,|\A|-1}\right)\bP\left(S_{y,0},\cdots,S_{y,|\A|-1}\right) &\\
&\leq  \sum_{ \forall z, S_{yz}: n^S_{yz}\ge\frac{n\Q_{yz}}2} 2 \exp \left( -\frac{2 t^2}{\left(\frac{C_1}{n^{S}_{yz}}\right)^4  n^3 } \right) \bP\left(S_{y,0},\cdots,S_{y,|\A|-1}\right) \\
&\overset{(a)}{\leq} 2 \exp \left( -\frac{2 t^2}{\left(\frac{2C_1}{n\Q_{yz}}\right)^4  n^3 } \right) = 2 \exp \left( -2 t^2 n \left(\frac{\Q_{yz}}{2C_1}\right)^4\right)
\end{flalign*}
step $(a)$ is by noting that the inner quantity is maximized when $n_{yz}^S= \frac{n\Q_{yz}}{2}$, combining things:
\begin{flalign*}
\nonumber\bP\left(|\hat{G}_z^{-1}q_{y,.}^S-\gamma_{yz}|>t\right)&\leq
  |A|  \exp{\left(- \frac{\min_{a}n\Q_{ya}}8\right)} +2 \exp \left( -2 t^2 n \left(\frac{\Q_{yz}}{2C_1}\right)^4\right)
\end{flalign*}
Now if $n \geq \frac{8\log(|\A|/\delta)}{\min_{yz}n\Q_{yz}}$ and $t \geq \sqrt{\frac{\log(2/\delta)}{2n}} \frac{4C_1^2}{\min_{yz}\Q_{yz}^2}$ then we have:
\begin{flalign*}
\nonumber\bP\left(|\hat{G}_z^{-1}q_{y,.}^S-\gamma_{yz}|>t\right)&\leq
  \delta + \delta
\end{flalign*}

\textbf{Step 3:} \textit{Projecting the estimator onto the simplex }.

One issue that arises is that our estimator for $\gamma_{y,z}$ does not lie in the range $[0,1]$, and hence we have to project the whole vector onto the simplex for it to be valid; note that this is not required if we are only interested in differences i.e. computing discrimination.
Our estimator for the vector of conditional probabilities for $a \in \mathcal{A}$ of $\bP(\hY=1|Y=y,A=a)$ is $\textrm{Proj}_{\Delta}(\hat{G}^{-1}q_{y,.})$ where $\textrm{Proj}_{\Delta}(x)$ is the orthogonal projection of $x$ onto the simplex defined as:
\begin{align*}
\textrm{Proj}_{\Delta}(\mathbf{x}) := \quad& \arg\min_{\mathbf{y}} \frac{1}{2} \left|| \mathbf{y}-\mathbf{x} \right||_2^2 \\
& \textrm{s.t.} \  \mathbf{y}^T\mathbf{1} = 1 , \ \mathbf{y} \geq 0
\end{align*}
The above problem can be solved optimally in a non-iterative manner in time $\mathcal{O}\left(|\mathcal{A}|\log(|\mathcal{A}|\right)$ (\cite{duchi2008efficient}). Denote by $\mathbf{x'}= \textrm{Proj}_{\Delta}(\mathbf{x})$, then by the definition of the projection for any $\mathbf{y} \in \Delta^{|\mathcal{A}|}$:
\[
|\mathbf{x'}- \mathbf{y}| \le |\mathbf{x}- \mathbf{y}| 
\]
however it does not hold that $||\mathbf{x'}- \mathbf{y}||_{\infty} \le ||\mathbf{x}- \mathbf{y}||_{\infty}$,  but :  $||\mathbf{x'}- \mathbf{y}||_{\infty} \le |\mathcal{A}| \cdot ||\mathbf{x}- \mathbf{y}||_{\infty}$.
Therefore:
\begin{flalign}
\nonumber &\bP\left(  \left| \textrm{Proj}_{\Delta}(\hat{G}^{-1}_k q_{y,.}^S) - \bP(\hY=1|Y=y,A=k)\right| >t \right) \le \bP\left( \max_{k} \left|\hat{G}^{-1}_k q_{y,.}^S - \bP(\hY=1|Y=y,A=k)\right| >\frac{t}{|\mathcal{A}|} \right) &\\ \nonumber
\end{flalign}

\textbf{Step 4:} \textit{Difference of Equalized odds}

Let $h_{ya} = \textrm{Proj}_{\Delta}(\hat{G}^{-1}_a q_{y,.}^S)$,  using a series of triangle inequality, 
\begin{equation*}
\left||h_{ya}^S-h_{y0}^S|-|h_{ya}-h_{y0}|\right|\le|h_{ya}^S-h_{y0}^S-h_{ya}+h_{y0}|\le|h^S_{ya}-h_{ya}|+|h^S_{y0}-h_{y0}|
\end{equation*}
hence
\begin{flalign*}
\nonumber\bP\left(\left||h_{ya}^S-h_{y0}^S|-|h_{ya}-h_{y0}|\right|> 2t\right) &
\le \bP\left(|h^S_{ya}-h_{ya}|+|h^S_{y0}-h_{y0}|>2t\right)&\\
 \nonumber &\overset{(a)}{\le}
 \bP\left(|h^S_{ya}-h_{ya}|>t\right) +  \bP\left(|h^S_{y0}-h_{y0}|>t\right) \\
 &\le 4 \delta
\end{flalign*}

where $(a)$ follows from union bound, and $(b)$ follows from above using  $n \geq \frac{8\log(|\A|/\delta)}{\min_{yz}n\Q_{yz}}$ and $t \geq \sqrt{\frac{\log(2/\delta)}{2n}} \frac{4|\A|C_1^2}{\min_{yz}\Q_{yz}^2}$
The lemma follows from collecting the failure probabilities for $y=0,1$, re-scaling $\delta$ and noting that $\min_{yz}\Q_{yz} \geq \min_{ya} \P_{ya}$.

Now let us write $t$ in terms of $\epsilon$, we write each of the factors involving $\pi$ in terms of $\epsilon$:

\begin{flalign*}
&C_1 = \frac{\pi +|\A|-2}{|\A| \pi -1} = \frac{|\A| -2 + e^{\epsilon}}{e^{\epsilon} -1} &
\end{flalign*}
and:
\begin{flalign*}
  C_1^2 &=  \frac{e^{2\epsilon} + 2(|\A|-2)e^\epsilon + (|\A|-2)^2      }{e^{2\epsilon} -2 e^{\epsilon} +1}\leq \frac{2|\A|^2 e^{2\epsilon}}{e^{2\epsilon} -2 e^{\epsilon} +1}&
\end{flalign*}

\end{proof}

\subsection{Section \ref{sec:learning}}
\subsubsection{First Step Algorithm Details}
Recall that in Algorithm \ref{alg:reductions}, the learner's best response gaced a given vector $\boldsymbol{\lambda}$ ($\textrm{BEST}_h(\boldsymbol{\lambda})$) puts all the mass on a single predictor $h \in \HC$ as the langragian  $L$ is linear in $Q$. \cite{agarwal2018reductions} shows that finding the learner's best response amounts to solving a cost sensitive classification problem. We now re-establish this reduction:
\begin{flalign}
L(h,\boldsymbol\lambda) &= \hat{\err}(h) + \boldsymbol\lambda^\top (M\boldsymbol\gamma(h) - \alpha_n \mathbf{1}) \nonumber&\\
&= \frac{1}{n}  \sum_{i \in S} \mathbb{I}_{h(x_i)\neq y_i}  - \alpha_n \boldsymbol\lambda^\top \mathbf{1} + \sum_{k,j} M_{k,j} \lambda_k \gamma_j^S(h) \nonumber  \\
&=  - \alpha_n \boldsymbol\lambda^\top  \mathbf{1} + \frac{1}{n} \sum_{i \in S} \mathbb{I}_{h(x_i)\neq y_i}  + \sum_{k,j} M_{k,j} \lambda_k \frac{1/n \cdot  h(x_i) \mathbb{I}_{(y_i,a_i)=j} }{1/n \sum_{s \in S}  \mathbb{I}_{(y_s,a_s)=j}}    \label{eq:langragian_expanded} 
\end{flalign}
Thus from equation \eqref{eq:langragian_expanded} and expanding the form of the matrix $M$ we have that minimizing $L(h,\lambda)$ over $h \in \HC$ is equivalent to solving a cost sensitive classification problem on $\{(x_i,c_i^0,c_i^1)\}_{i=1}^n$ where the costs are:
\begin{flalign*}
&c_i^0 = \mathbb{I}_{y_i \neq 0} &\\
&c_i^1 =  \mathbb{I}_{y_i \neq 1} + \frac{\lambda_{(a_i,y_i,+)} - \lambda_{(a_i,y_i,-)} }{p^S_{a_i,y_i}} \mathbb{I}_{a_i \neq 0} - \sum_{a \in \mathcal{A}\setminus \{0\}} \frac{\lambda_{(a,y_i,+)} - \lambda_{(a,y_i,-)} }{p^S_{0,y_i}}
\end{flalign*}
where $p^S_{a,y} = \frac{1}{n}\sum_{s \in S}  \mathbb{I}_{(y_s=y,a_s=a)}$.

The goal of Algorithm \ref{alg:reductions} is to return for any degree of approximation $\nu \in \mathbb{R}^+$ a $\nu$-approximate saddle point $(\hat{Q},\hat{\lambda} )$ defined as:

\begin{align}
&L(\hat{Q},\hat{\boldsymbol\lambda}) \leq L(Q,\hat{\boldsymbol\lambda}) + \nu  \ \ \forall Q \in  \Delta_{\HC} \\
&L(\hat{Q},\hat{\boldsymbol\lambda}) \geq L(\hat{Q},\boldsymbol\lambda ) - \nu \ \ \forall \boldsymbol\lambda \in \mathbb{R}^{|\mathcal{K}|}_+ , ||\boldsymbol\lambda||_1 \leq B
\end{align}
From Theorem 1 in \cite{agarwal2018reductions}, if we run the algorithm for at least $\frac{16 \log(4 |\mathcal{A}|+1)}{\nu^2}$ iterations with learning rate $\eta = \frac{\nu}{8 B}$ it returns a $\nu$-approximate saddle point.
\subsubsection{First Step Guarantees}

\begin{lemma}\label{lemma:q_concentration}
Denote by $\mathbf{Q}_{yz}=\bP(Y=y,Z=z)$, $q_{yz}(\hY)=\bP(\hY=1|Y=y,Z=z)$, for $\delta \in (0,1/2)$ and $h$ any binary predictor, if $n \geq \frac{8\log{8 |\A|/\delta}}{\min_{yz} Q_{yz}} $, then:
\begin{flalign}
\nonumber\bP\left(\max_{ya} \left| |q_{ya}^S-q_{y0}^S|-|q_{ya}-q_{y0}|\right|> 2\sqrt{\frac{\log{16 |\A|/\delta}}{n \min_{yz}\Q_{yz}}}\right) \le \delta
\end{flalign}
\end{lemma}
\begin{proof}
Let $a \in \A$,	
 denote $\mathbf{Q}_{yz}=\bP(Y=y,Z=z)$, $q_{yz}(\hY)=\bP(\hY=1|Y=y,Z=z)$, then by \cite{woodworth2017learning} or step 1 of Lemma 1:
\begin{flalign*}
\nonumber\bP\left(|q_{yz}^S-q_{yz}|>t\right)&\le \exp{\big(- \frac{n\Q_{yz}}8\big)}  +2\exp{(-t^2n\Q_{yz})}&
\end{flalign*}

Now  using a series of triangle inequality identical to step 4 of Lemma 1, 
\begin{equation*}
\left||q_{ya}^S-q_{y0}^S|-|q_{ya}-q_{y0}|\right|\le|q_{ya}^S-q_{y0}^S-q_{ya}+q_{y0}|\le|q^S_{ya}-q_{y0}|+|q^S_{y0}-q_{y0}|
\end{equation*}
hence
\begin{flalign*}
\nonumber\bP\left(\left||q_{ya}^S-q_{y0}^S|-|q_{ya}-q_{y0}|\right|> 2t\right) &
\le \bP\left(|q^S_{ya}-q_{y0}|+|q^S_{y0}-q_{y0}|>2t\right)&\\
 \nonumber &\overset{(a)}{\le}
 \bP\left(|q^S_{ya}-q_{y0}|>t\right) +  \bP\left(|q^S_{y0}-q_{y0}|>t\right) \\
 &\le 2\exp{\big(- \frac{n \min_{yz} \Q_{yz}}8\big)}  +4\exp{(-t^2n \min_{yz}\Q_{yz})}\\
 &\overset{(b)}{\le} \frac{\delta}{2 |\A|}
\end{flalign*}

where $(a)$ follows from union bound, and $(b)$ follows if $n \geq \frac{8\log{8 |\A| /\delta}}{\min_{yz} Q_{yz}} $ and $t = \sqrt{\frac{\log{16 |\A| /\delta}}{n \min_{yz}\Q_{yz}}}$

The lemma follows from collecting the failure probabilities for $y=0,1$ and $\forall a \in \mathcal{A}$.
\end{proof}

\begin{lemma}\label{lemma:relating_disc_a_z}
If a binary predictor $\hat{Y}$ is independent of $Z$ given $A$, then if the groups are binary it holds that:

\begin{equation}
q_{y1}(\hat{Y}) - q_{y0}(\hat{Y})  =  \left(\gamma_{y1}(\hat{Y}) -\gamma_{y0}(\hat{Y}) \right) \frac{ (2\pi -1)\P_{y1} \P_{y0} }
	{\Q_{y1} \Q_{y0} }  \label{eq:relategamma-q}
\end{equation}
For general $|\A|$ different groups, we have $\forall k,j \in \A$ the following relation:
 \begin{flalign}
   \nonumber |\gamma_{y,k} - \gamma_{y,j}|&\leq5 C \frac{\max_i \bP(Z=i,Y=y)}{\min_j \bP(A=j,Y=y)^2}  \left|\max_z q_{y,z} - \min_{z'} q_{y,z'}\right|
\end{flalign}
 where $C = \frac{\pi + |\mathcal{A}|-2}{|\A|\pi -1    } $.

\end{lemma}
\begin{proof}
We begin by noting the following relationship established in step 4 of Lemma 1:
\begin{flalign}
\nonumber	\bP(\hat{Y}=1 |Y=y, Z =a)
	&= \pi \bP(\hat{Y}=1 |Y=y,A =a) \frac{\bP(Y=y,A=a)}{\bP(Z=a,Y=y)} 
	\\ \nonumber&+ \sum_{a'\setminus a} \bar{\pi} \bP(\hat{Y}=1 |Y=y,A =a') \frac{\bP(Y=y,A=a')}{\bP(Z=a,Y=y) }  
\end{flalign}

From the above equation, we can evaluate for any $a,b \in \A$ the difference between $q_{ya}$ and $q_{yb}$ in terms of $\gamma_{y.}$, denoting $\P_{ya} =\bP(Y=y,A=a)$ :

\begin{flalign*}
q_{ya} - q_{yb} 
	&= \gamma_{ya} \frac{\pi\P_{ya}}{\Q_{ya}} 
	+ \sum_{a'\setminus a} \gamma_{ya'} \frac{\bar{\pi} \P_{ya'}
}{\Q_{ya}} 
	- \gamma_{yb} \frac{\pi \P_{yb}}{ \Q_{yb}} 
	- \sum_{b'\setminus b} \gamma_{yb'} \frac{\bar{\pi} \P_{yb'}}{\Q_{yb}
} \nonumber &\\
&= \frac{  \gamma_{ya} \pi\P_{ya}\Q_{yb}  	+  \gamma_{yb} \bar{\pi}\P_{yb}\Q_{yb}  + \sum_{a'\setminus \{a,b\}} \bar{\pi} \gamma_{ya'} \P_{ya'} \Q_{yb}   }{   \Q_{ya} \Q_{yb}     } \nonumber\\
&- \frac{  \gamma_{yb} \pi\P_{yb}\Q_{ya}  	+  \gamma_{ya} \bar{\pi}\P_{ya}\Q_{ya}  + \sum_{b'\setminus \{a,b\}} \bar{\pi}  \gamma_{yb'} \P_{yb'} \Q_{ya}   }{   \Q_{ya} \Q_{yb}     }\nonumber \\
&= \frac{\gamma_{ya}\P_{ya} (\pi  \Q_{yb} - \bar{\pi} \Q_{ya}) - \gamma_{yb}\P_{yb}(\pi  \Q_{ya} -  \bar{\pi}  \Q_{yb})          }{\Q_{ya} \Q_{yb}} + \frac{\sum_{c\setminus \{a,b\}} \bar{\pi}\gamma_{yc} \P_{yc} (\Q_{yb} - \Q_{ya})}{\Q_{ya} \Q_{yb}}
\\
&\overset{(a)}{=} \frac{\gamma_{ya}\P_{ya} (\pi (\pi \P_{yb} + \bar{\pi} \P_{ya} + \bar{\pi} \sum_{c \setminus \{a,b\}} \P_{yc}) - \bar{\pi} (\pi \P_{ya} + \bar{\pi} \P_{yb} + \bar{\pi} \sum_{c \setminus \{a,b\}} \P_{yc})}{\Q_{ya}\Q_{yb}}  \\&- \frac{\gamma_{yb}\P_{yb}(\pi (\pi \P_{ya} + \bar{\pi} \P_{yb} + \bar{\pi} \sum_{c \setminus \{a,b\}} \P_{yc}) - \bar{\pi} (\pi \P_{yb} + \bar{\pi} \P_{ya} + \bar{\pi} \sum_{c \setminus \{a,b\}} \P_{yc}))          }{\Q_{ya} \Q_{yb}} \\&+ \frac{\sum_{c\setminus \{a,b\}} \bar{\pi}\gamma_{yc} \P_{yc} (\Q_{yb} - \Q_{ya})}{\Q_{ya} \Q_{yb}}
\\
&=\frac{\gamma_{ya}\P_{ya} (\pi^2 \P_{yb} - \bar{\pi}^2 \P_{yb} + (\pi-\bar{\pi})\bar{\pi} \sum_{c \setminus \{a,b\}} \P_{yc})}{\Q_{ya}\Q_{yb}}  \\&- \frac{\gamma_{yb}\P_{yb}(\pi^2 \P_{ya} - \bar{\pi}^2 \P_{ya} + (\pi-\bar{\pi})\bar{\pi} \sum_{c \setminus \{a,b\}} \P_{yc})}{\Q_{ya} \Q_{yb}} + \frac{\sum_{c\setminus \{a,b\}} \bar{\pi}\gamma_{yc} \P_{yc} (\Q_{yb} - \Q_{ya})}{\Q_{ya} \Q_{yb}}\\
	&= \frac{(\gamma_{ya}-\gamma_{yb}) \P_{ya} \P_{yb} (\pi^2 - \bar{\pi}^2)  }
	{\Q_{ya} \Q_{yb} } \nonumber \\&
    +\frac{(\gamma_{ya}\P_{ya}-\gamma_{yb}\P_{yb}) (\pi-\bar{\pi})\bar{\pi} \sum_{c \setminus \{a,b\}} \P_{yc}  }
	{\Q_{ya} \Q_{yb} } + \frac{\sum_{c\setminus \{a,b\}} \bar{\pi}\gamma_{yc} \P_{yc} (\Q_{yb} - \Q_{ya})}{\Q_{ya} \Q_{yb}}
\end{flalign*}
where step (a) follows from expanding by equation \eqref{eq:Qyz}.  If $\A = \{0,1\}$ then the above reduces to:

\begin{equation*}
q_{y1} - q_{y0} =  (\gamma_{y1}-\gamma_{y0}) \frac{ (2\pi -1)\P_{y1} \P_{y0} }
	{\Q_{y1} \Q_{y0} } 
\end{equation*}
Now when the groups are not binary we instead rely on an upper bound.

Let $q_y = [\bP(\hat{Y}=1|Y=y,Z=0),\cdots,\bP(\hat{Y}=1|Y=y,Z=|\A|-1)]^\top$, $\gamma = [\bP(\hat{Y}=1|Y=y,A=0),\cdots,\bP(\hat{Y}=1|Y=y,A=|\A|-1)]^\top $, in the proof of Lemma 1 we established that $G^{-1}q_y= \gamma_y$, now let $k,j \in \A$ then we have:
\begin{flalign}
   \nonumber |\gamma_{y,k} - \gamma_{y,j}|&= |G^{-1}_kq_y - G^{-1}_jq_y | &\\ \nonumber
    &\overset{(a)}{=}| G^{-1}_k(q_y - q') - G^{-1}_j(q_y-q')|&\\ \nonumber
    &=|(q_y - q')( G^{-1}_k -G^{-1}_j)| \\ \nonumber
    &\leq |q_y - q'|_{\infty} | G^{-1}_k -G^{-1}_j|_1  \ \textrm{(Holder's Inequality)} \\ \label{lemma:relateztoa_prelim}
    &= |\max_z q_{y,z} - \min_{z'} q_{y,z'}| \cdot | G^{-1}_k -G^{-1}_j|_1 
\end{flalign}
in step $(a)$ we introduce $q'=[\min_{z}q_{y,z},\cdots, \min_{z}q_{y,z}]^\top$, and note that $G^{-1}_k q' = \min_{z}q_{y,z} $ as the rows of $G$ sum to $1$ by the proof of Proposition \ref{prop:forzeqfora}, therefore $G^{-1}_k q' = G^{-1}_j q'$ and the difference in the previous step is unchanged. Now let us take expand the right most term in equation \eqref{lemma:relateztoa_prelim}, for ease of notation let  $\bP(Z=i,Y=y)=z_i$ and $\bP(A=i,Y=y)=a_i$:
\begin{flalign*}
&| G^{-1}_k -G^{-1}_j|_1 &\\&= \sum_{a \in \A \setminus \{k,j \}} \left|C_2(\frac{z_a}{a_k} - \frac{z_a}{a_i} )  \right| + \left|C_1\frac{z_k}{a_k} -C_2 \frac{z_k}{a_i}   \right| + \left|C_2\frac{z_i}{a_k} -C_1 \frac{z_i}{a_i}   \right|
\\&= \sum_{a \in \A \setminus \{k,j \}} \left|C_2(\frac{z_a a_i - z_a a_k}{a_k a_i}  )  \right| + \left|C_1\frac{z_k a_i}{a_k a_i} -C_2 \frac{z_k a_k}{a_k a_i}   \right| + \left|C_2\frac{z_i a_i}{a_k a_i} -C_1 \frac{z_i a_k}{a_k a_i}   \right|
\\&= \sum_{a \in \A \setminus \{k,j \}} \left|C_2\frac{z_a (a_i - a_k)}{a_k a_i}    \right| + \left|\frac{z_k (C_1a_i-C_2 a_k)}{a_k a_i}   \right| + \left|\frac{z_i(C_2 a_i - C_1 a_k)}{a_k a_i}   \right|
\\&\leq \max_a z_a \cdot \left( (|\A|-2) \left|C_2\frac{ (a_i - a_k)}{a_k a_i}    \right| + \left|\frac{ (C_1a_i-C_2 a_k)}{a_k a_i}   \right| +\left|\frac{(C_2 a_i - C_1 a_k)}{a_k a_i}   \right| \right)
\\&\leq \max_a z_a \cdot \left( (|\A|-2) \left|C_2\frac{ 1}{\min_z a_z^2}    \right| + \left|2C_1\frac{ 1}{\min_z a_z^2}   \right| +\left|2C_1\frac{ 1}{\min_z a_z^2}   \right| \right)
\\&\leq \frac{\max_a z_a}{\min_z a_z^2}  \cdot \left( (|\A|-2) \left|C_2    \right| + 4 C_1 \right) \leq \frac{\max_i \bP(Z=i,Y=y)}{\min_j \bP(A=j,Y=y)^2} 5 C_1
\end{flalign*}
 Hence we have the following inequality:
 \begin{flalign}
   \nonumber |\gamma_{y,k} - \gamma_{y,j}|&\leq5 C_1 \frac{\max_i \bP(Z=i,Y=y)}{\min_j \bP(A=j,Y=y)^2}  \left|\max_z q_{y,z} - \min_{z'} q_{y,z'}\right|
\end{flalign}
 where $C_1 = \frac{\pi + |\mathcal{A}|-2}{|\A|\pi -1    } $.

\end{proof}

We now recall some helper lemmas from \cite{agarwal2018reductions}.

\begin{lemma}[Lemma 2 \cite{agarwal2018reductions}]
	For any distribution $Q$ satisfying the empirical constraints on dataset $S$: $M \gamma^S(Q) \leq \alpha_n \mathbf{1}$, $\hat{Q}$ the output $\hat{Y}$ of Algorithm \ref{alg:reductions} satisfies:
	\begin{equation}
	\err^S(\hat{Y}) \leq \err(Q) + 2 \nu 
	\end{equation}
	\label{lemma:2agarwal}
\end{lemma}

\begin{lemma}[Lemma 3 \cite{agarwal2018reductions}]
	The discrimination of $\hat{Y}$, output of Algorithm \ref{alg:reductions}, satisfies:
	\begin{equation}
	\max_{y,a} | q_{y,a}^S(\hat{Y}) - q_{y,0}^S(\hat{Y})| \leq 2\alpha_n + 2\frac{1 + 2 \nu}{B}
	\end{equation}
	\label{lemma:3agarwal}
	
\end{lemma}

\textbf{Lemma \ref{lemma:step1_guarantees}} [Guarantees for Step 1] \textit{ \noindent 	Given a hypothesis class $\mathcal{H}$, a distribution over $(X,A,Y)$, $B \in \mathbb{R}^+$ and any $\delta \in (0,1/2)$, then with probability greater than $1- \delta$, if  $n \geq \frac{16\log{8 |\A|/\delta}}{\min_{ya} \P_{ya}}$, $\alpha_n = 2 \sqrt{\frac{\log{|\A|/\delta}}{n \min_{ya}\P_{ya}}}$ and we let $\nu  = \mathfrak{R}_{n/2}(\mathcal{H}) + \sqrt{\frac{\log{8/\delta}}{n}}$, then running Algorithm 1 on dataset $S$ with $T \geq \frac{16 \log(4 |\mathcal{A}|+1)}{\nu^2}$ and learning rate $\eta = \frac{\nu}{8 B}$ returns a predictor $\hat{Y}$ satisfying the following:
	\begin{align*}
	&\err(\hat{Y}) \leq_{\delta/2} \err(Y^*)  + 4 \mathfrak{R}_{n/2}(\mathcal{H}) + 4 \sqrt{\frac{\log{1/\delta}}{n}} \\
	&\disc(\hat{Y}) \leq_{\delta/2}  \frac{5C}{\min_{ya} \mathbf{P}_{ya}^2}  \left( \frac{2}{B}  + 6 \mathfrak{R}_{\frac{\min_{ya}n \P_{ya}}{4}}(\mathcal{H}) + 10 \sqrt{\frac{2\log{64 |\A|/\delta}}{n \min_{ya}\P_{ya}}} \right)
	\end{align*}
}

\begin{proof}
From Theorem 1 in \cite{agarwal2018reductions}, if we run the algorithm for at least $\frac{16 \log(4 |\mathcal{A}|+1)}{\nu^2}$ iterations with learning rate $\eta = \frac{\nu}{8 B}$ it returns a $\nu$-approximate saddle point. We set $\nu$ at the end of the proof to balance the bounds.

	For step 1 we have access to $S_1=\{(x_i,y_i,z_i)\}_{i=1}^{n/2}$,
	denote by $\err(\hat{Y})=\bP(\hat{Y} \neq Y)$, using the Rademacher complexity bound (Theorem 3.5 \cite{mohri2018foundations}) and the fact that $\mathfrak{R}_n(\Delta_{\HC})=\mathfrak{R}_n(\HC)$ we have:
	\begin{equation}\label{eq:error_concentration}
	\err(\hat{Y}) \leq_{\delta/4} \err^S(\hat{Y}) + \mathfrak{R}_{n/2}(\mathcal{H}) + \sqrt{\frac{\log{8/\delta}}{n}}
	\end{equation}
	Now from Lemma 5 of \cite{woodworth2017learning}, with probability greater than $1-\delta/4$, $Y^*$ is in the feasible set of step 1 if $\alpha_n\ge2 \sqrt{\frac{2\log{64 |\A| /\delta}}{n \min_{yz}\Q_{yz}}}$, hence we can apply Lemma \ref{lemma:2agarwal} with $Y^*$ and the concentration bound \eqref{eq:error_concentration} :
	\begin{equation*}
	\err(\hat{Y}) \leq_{\delta/2} \err(Y^*) + 2\nu +2 \mathfrak{R}_{n/2}(\mathcal{H}) +2 \sqrt{\frac{\log{8/\delta}}{n}}
	\end{equation*}
	
	For the constraint, from Lemma \ref{lemma:q_concentration}, if $n \geq \frac{16\log{8 |\A|/\delta}}{\min_{yz} \Q_{yz}} $, then 
	\begin{flalign*}
	\max_{ya}\left||q_{ya}^S-q_{y0}^S|-|q_{ya}-q_{y0}|\right| \leq_{\delta/4}  2 \sqrt{\frac{2\log{64 |\A|/\delta}}{n \min_{yz}\Q_{yz}}}
	\end{flalign*}
	Similarly from the standard Rademacher complexity bound (Theorem 3.3 \cite{mohri2018foundations}) and since our function class for the constraint is $\mathcal{H}$ it holds that (by Lemma 6 \cite{agarwal2018reductions}):
	\begin{flalign*}
	\max_{ya}|q_{ya}-q_{y0}| \leq_{\delta/4} |q_{ya}^S-q_{y0}^S| + 2 \mathfrak{R}_{\frac{\min_{yz}n \Q_{yz}}{4}}(\mathcal{H}) + 2 \sqrt{\frac{2\log{64 |\A|/\delta}}{n \min_{yz}\Q_{yz}}}
	\end{flalign*}
	Applying Lemma \ref{lemma:3agarwal}:
	\begin{equation}
	|q_{ya}^S-q_{y0}^S|  \leq 2\alpha_n + 2\frac{1 + 2\nu}{B}
	\end{equation}
	Combining things with $\alpha_n =2 \sqrt{\frac{2\log{64 |\A| /\delta}}{n \min_{yz}\Q_{yz}}}$ :
	\begin{flalign*}
	\max_{ya}|q_{ya}-q_{y0}| \leq_{\delta/4}  \frac{2 + 4 \nu}{B}  + 2 \mathfrak{R}_{\frac{\min_{yz}n \Q_{yz}}{4}}(\mathcal{H}) + 6 \sqrt{\frac{2\log{64 |\A|/\delta}}{n \min_{yz}\Q_{yz}}}
	\end{flalign*}

	Now by Lemma \ref{lemma:relating_disc_a_z} we can re-state the above in terms of $A$:
	\begin{flalign*}
	\max_{a}|\gamma_{y,a} - \gamma_{y,0}| \leq_{\delta/4} \frac{5C}{\min_{ya} \mathbf{P}_{ya}^2}  \left( \frac{2 + 4 \nu}{B}  + 2 \mathfrak{R}_{\frac{\min_{yz}n \Q_{yz}}{4}}(\mathcal{H}) + 6 \sqrt{\frac{2\log{64 |\A|/\delta}}{n \min_{yz}\Q_{yz}}} \right)
	\end{flalign*}

	For simplicity, we can thus set $\nu =  \mathfrak{R}_{n/2}(\mathcal{H}) + \sqrt{\frac{\log{8/\delta}}{n}}$, by noting that $\min_{yz} \Q_{yz} \geq \min_{ya} \P_{ya}$ we obtain the lemma statement. 
	
\end{proof}

\subsubsection{Second Step Algorithm Details}

Given a predictor $\hat{Y}$, Hardt et al. give a simple procedure to obtain a derived predictor $\tilde{Y}$ that is non-discriminatory \cite{hardt2016equality} by solving a constrained linear program (LP). One of the caveats of the approach is that it requires the use of the protected attribute at test time, and in our setting we do not have access to $A$ but $Z$. We have seen in section \ref{sec:auditing} that predictors that rely on $Z$ cannot be trusted even if they are completely non-discriminatory with respect to the privatized attribute. Despite this difficulty, it turns out if the base predictor $\hat{Y}$ is independent of $Z$ given $A$, then we can re-write the LP to obtain a derived predictor $\tilde{Y}=h(\hat{Y},Z)$ that minimizes the error while being non-discriminatory with respect to $A$.

The approach boils down to solving the following linear program (LP):

\begin{align*}
	&\min \quad \bP(\tilde{Y} \neq Y) \\
	& s.t. \quad \bP(\tilde{Y} =1 |A = a, Y =y) = \bP(\tilde{Y} =1 | Y =y, A=0) \\ & \quad \forall y \in \{ 0,1\}, \forall a \in \mathcal{A}
\end{align*}

We can write this objective by optimizing over $2 |\mathcal{A}|$ probabilities $p_{\hat{y},a} := \bP(\tilde{Y} =1 | \hat{Y}=\hat{y}, A =a)$ that completely specify the behavior of $\tilde{Y}$:

\begin{align}
	\tilde{Y} = &\arg\min_{p_{.,.}} \quad \sum_{\hat{y},a} (\bP(\hat{Y}=\hat{y}, A=a, Y =0) - \bP(\hat{Y}=\hat{y}, A=a, Y =1)) \cdot p_{\hat{y},a}  \label{eq:post-procesingobjective}   \\
	 s.t.&  \quad p_{0,a} \bP(\hat{Y} =0 | Y=y, A =a) + p_{1,a} \bP(\hat{Y} =1 | Y=y, A =a) \\ & =p_{0,0} \bP(\hat{Y} =0 | Y=y, A =0) + p_{1,0} \bP(\hat{Y} =1 | Y=y, A =0) , \quad \forall y \in \{ 0,1\}, \forall a \in \mathcal{A}  \label{eq:post-processingconstraint}\\
	 & 0 \leq p_{\hat{y},a} \leq 1 \quad \forall \hat{y} \in \{ 0,1\}, \forall a \in \mathcal{A}  \nonumber
\end{align}
Unfortunately we cannot directly solve the above program as we do not have access to $A$, however we can solve the problem with $Z$ replacing $A$; we denote this as the na\"ve program and as we have previously mentioned it cannot assure any degree of non-discrimination with respect to $A$.
Now let us see how we can transform this na\"ve program to satisfy equalized odds. We optimize over the set of variables that denote  $p_{\hat{y},z} := \bP(\tilde{Y} =1 | \hat{Y}=\hat{y}, Z =z)$. Now for the constraint note that $\bP(\tilde{Y} =1 | \hat{Y}=\hat{y}, A =a)$ can be expressed as a mixture of our decision variables:
\begin{flalign*}
	\bP(\tilde{Y} =1 | \hat{Y}=\hat{y}, A =a) 
	&= \sum_{a'} \bP(\tilde{Y} =1 | \hat{Y}=\hat{y}, Z =a', A=a) \bP(Z=a'|A=a, \hat{Y}=\hat{y}) &\\
	&= \pi \bP(\tilde{Y} =1 | \hat{Y}=\hat{y}, Z =a) +  \sum_{a'\setminus a} \hat{\pi} \bP(\tilde{Y} =1 | \hat{Y}=\hat{y}, Z =a')
\end{flalign*}
Since we assumed the base predictor $\hat{Y}$ is independent of $Z$ given $A$ then $\bP(\hat{Y} = \hat{y} | Y=y, A =a)$ can be recovered from the following linear system by using the same estimator we developed previously in Lemma \ref{lemma:concentration_of_estimator}:
\begin{flalign*}
\bP(\hat{Y} = \hat{y} | Y=y, A =a) &= \pi \bP(\hat{Y} = \hat{y} | Y=y, A =a) \frac{\bP(A=a,Y=y)}{\bP(Z=a,Y=y)}&\\& + \sum_{a' \neq a} \bar{\pi} \bP(\hat{Y} = \hat{y} | Y=y, A =a) \frac{\bP(A=a',Y=y)}{\bP(Z=a,Y=y)} 
\end{flalign*}
On the other hand for the objective we have:
\begin{equation*}
\bP(\hat{Y}=\hat{y}, Z=a, Y =y) = \pi \bP(\hat{Y}=\hat{y}, A=a, Y =y)  + \bar{\pi} \sum_{a'\neq a} \bP(\hat{Y}=\hat{y}, A=a', Y =y)
\end{equation*}
And hence our estimator for $\bP(\hat{Y}=\hat{y}, A=a, Y =y)$ is constructed by multiplying by the inverse of $\Pi$ and projecting onto the simplex.

Denote by $\tilde{p}_{\hat{y},a} = \pi p_{\hat{y},a}+  \sum_{a'\setminus a} \hat{\pi} p_{\hat{y},a'}$ and  $\tilde{\bP}^S(\hat{Y} = \hat{y} | Y=y, A =a)$ our estimator for $\bP(\hat{Y} = \hat{y} | Y=y, A =a)$ and similarly $\tilde{\bP}^S(\hat{Y} = \hat{y} , Y=y, A =a)$. We propose to solve the following optimization problem:
\begin{align}
\tilde{Y} = &\arg\min_{p_{.,.}} \quad \sum_{\hat{y},a} (\tilde{\bP}^S(\hat{Y}=\hat{y}, Z=a, Y =0) - \tilde{\bP}^S(\hat{Y}=\hat{y}, Z=a, Y =1)) \cdot \tilde{p}_{\hat{y},a}   \\
	 s.t.&  \quad \tilde{p}_{0,a} \tilde{\bP}^S(\hat{Y} = 0 | Y=y, A =a)+ \tilde{p}_{1,a}\tilde{\bP}^S(\hat{Y} = 1 | Y=y, A =a) \nonumber\\ & =\tilde{p}_{0,0}  \tilde{\bP}^S(\hat{Y} = 0 | Y=y, A =0) + \tilde{p}_{1,0}  \tilde{\bP}^S(\hat{Y} = 1 | Y=y, A =0) , \quad \forall y \in \{ 0,1\}, \forall a \in \mathcal{A}	   \\
	 & 0 \leq p_{\hat{y},a} \leq 1 \quad \forall \hat{y} \in \{ 0,1\}, \forall a \in \mathcal{A}  \nonumber
\end{align}

\subsubsection{Second Step Guarantees} \label{app:step-2}

\begin{lemma}[Step 2 guarantees]\label{lemma:step2-guarantees}
Let $\hat{Y}$ be a binary predictor that is independent of $Z$ given $A$, for any $\delta \in (0,1/2)$, if $n \ge \frac{32\log(8|\A|/\delta)}{\min_{ya}\P_{ya}}  $, $\tilde{\alpha}_n = \sqrt{\frac{\log(64/\delta)}{2n}} \frac{4 |\mathcal{A}|C^2}{\min_{ya}\P_{ya^*}^2}  $ and with $\tilde{Y}*$  an optimal 0-discriminatory predictor derived from $\hat{Y}$, then with probability greater than $1-\delta/2$ we have:

\begin{flalign}
 \err(\tilde{Y})    \leq \err(\tilde{Y}^*) +   4 |\mathcal{A}|C \sqrt{\frac{\log(32 |\mathcal{A}|/\delta)}{2n}}  \nonumber 
\end{flalign} 
\begin{equation*}
   \disc(\tilde{Y}) \le  \sqrt{\frac{ \log(\frac{64}{\delta})}{2n}} \frac{8 |\mathcal{A}|C^2}{\min_{ya}\P_{ya}^2} 
\end{equation*}

\end{lemma}
\begin{proof}
Denote $\err(\tilde{Y})= \bP(\tilde{Y} \neq Y)$ and $q_{\hat{y},a,y}:=\bP(\hat{Y}=\hat{y}, Z=a, Y =1)$ , then for any $\tilde{Y}$ in the derived set of $\hat{Y}$ (also by Lemma B.2 \cite{jagielski2018differentially}):

\begin{flalign}
\left| \err^S(\tilde{Y}) - \err(\tilde{Y})    \right|& = \left|  \sum_{\hat{y},a} \tilde{p}_{\hat{y},a} \cdot \left( (\tilde{q}^S_{\hat{y},a,0} - q_{\hat{y},a,0}) + (q_{\hat{y},a,1}- \tilde{q}^S_{\hat{y},a,1})   \right) \right|& \nonumber\\
&\leq  |\sum_{\hat{y},a}   \tilde{q}^S_{\hat{y},a,0} - q_{\hat{y},a,0}   | + |\sum_{\hat{y},a} \ q_{\hat{y},a,1}- \tilde{q}^S_{\hat{y},a,1}   | \nonumber \\
& \leq    \sum_{\hat{y},a}   |\tilde{q}^S_{\hat{y},a,0} - q_{\hat{y},a,0}   | + \sum_{\hat{y},a} |\ q_{\hat{y},a,1}- \tilde{q}^S_{\hat{y},a,1}   | \label{diff:errtilde-err}
\end{flalign}

Now our estimator $\tilde{q}^S_{\hat{y},a,y}$ for $q_{\hat{y},a,y}$ is obtained by multiplying by the inverse of the matrix $\Pi$ and projecting onto the simplex, as was done in Lemma 1. Using the same arguments of step 2 of the proof of Lemma 1 using Mcdirmid's inequality we have:
\begin{flalign}
    \bP\left( |\tilde{q}^S_{\hat{y},a,y} -q_{\hat{y},a,y} |>t\right) \le 2 \exp(- \frac{2t^2 n}{C^2}) \nonumber
\end{flalign}
Hence
\begin{flalign}
\bP(\left| \err^S(\tilde{Y}) - \err(\tilde{Y})    \right|>t)&  \leq   \bP( \sum_{\hat{y},a}   |q^S_{\hat{y},a,0} - q_{\hat{y},a,0}   | + \sum_{\hat{y},a} |\ q_{\hat{y},a,1}- q^S_{\hat{y},a,1}   |>t) \nonumber \\
&\leq 8 |\A| \exp(-2n \left(\frac{t}{4|\A| } \frac{|\A| \pi -1}{\pi + |\A| -2}\right)^2) 
\end{flalign}
Thus if $t \geq \frac{4|\A|(\pi + |\A| -2)}{|\A| \pi -1} \sqrt{\frac{\log(32 |\A|/\delta)}{2n}}$ : \begin{flalign}
\bP\left(\left| \err^S(\tilde{Y}) - \err(\tilde{Y})    \right|> \frac{4|\A|(\pi + |\A| -2)}{|\A| \pi -1} \sqrt{\frac{\log(32 |\A|/\delta)}{2n}}\right)&  \leq   \delta/4 \nonumber 
\end{flalign} 

Now for the fairness constraint, denote $\Gamma_{y,a}(\tilde{Y}) = |\bP(\tilde{Y} =1 | Y=\hat{y}, A =a) - \bP(\tilde{Y} =1 | Y=y, A = 0)|$, then:
\begin{flalign*}
&|\tilde{\Gamma}_{y,a}^S(\tilde{Y}) - \Gamma_{y,a}(\tilde{Y})| \ =  &\\&
|\tilde{p}_{0,a} (1-\tilde{\bP}^S(\hat{Y} = 1 | Y=y, A =a))+ \tilde{p}_{1,a}\tilde{\bP}^S(\hat{Y} = 1 | Y=y, A =a)\\& -\tilde{p}_{0,0}  (1-\tilde{\bP}^S(\hat{Y} = 1 | Y=y, A =0)) - \tilde{p}_{1,0}  \tilde{\bP}^S(\hat{Y} = 1 | Y=y, A =0) \\
&-\tilde{p}_{0,a} (1-\tilde{\bP}(\hat{Y} = 1 | Y=y, A =a))- \tilde{p}_{1,a}\tilde{\bP}(\hat{Y} = 1 | Y=y, A =a)\\& +\tilde{p}_{0,0}  (1-\tilde{\bP}(\hat{Y} = 1 | Y=y, A =0)) + \tilde{p}_{1,0}  \tilde{\bP}(\hat{Y} = 1 | Y=y, A =0) | \\
& \le |\tilde{\bP}^S(\hat{Y} = 1 | Y=y, A =a) - \tilde{\bP}(\hat{Y} = 1 | Y=y, A =a)| \cdot |\tilde{p}_{1,a} - \tilde{p}_{0,a}| \\&+ |\tilde{\bP}^S(\hat{Y} = 1 | Y=y, A =0) - \tilde{\bP}(\hat{Y} = 1 | Y=y, A =0)| \cdot |\tilde{p}_{1,0} - \tilde{p}_{0,0}| \\
& \le |\tilde{\bP}^S(\hat{Y} = 1 | Y=y, A =a) - \tilde{\bP}(\hat{Y} = 1 | Y=y, A =a)|  \\&+ |\tilde{\bP}^S(\hat{Y} = 1 | Y=y, A =0) - \tilde{\bP}(\hat{Y} = 1 | Y=y, A =0)| 
\end{flalign*}

From the proof of Lemma 1, let $C= \frac{\pi + |\A|-2}{|\A|\pi-1}$, then if
$n \geq \frac{32\log(8|\A|/\delta)}{\min_{ya}\P_{ya}}$, we have:
 \begin{equation*}
     \bP\left( \max_{ya}| \tilde{\Gamma}_{ya}^S - \Gamma_{ya}| >
     \ \sqrt{\frac{\log(64/\delta)}{2n}} \frac{4|\A|C^2}{\min_{ya}\P_{ya}^2}
     \right) \leq \delta/4
 \end{equation*}

Now if $\tilde{\alpha}_n \ge \sqrt{\frac{\log(64/\delta)}{2n}} \frac{4|\A|C^2}{\min_{ya}\P_{ya}^2}$, then by the same argument of Lemma 5 in \cite{woodworth2017learning}, any 0-discriminatory $\tilde{Y}^*$ derived from $\hat{Y}$ is in the feasible set of step 2 with probability greater than $1-\delta/4$, hence by the optimality of $\tilde{Y}$ on $S_2$:
\begin{flalign}
 \err(\tilde{Y})    \leq_{\delta/2}   \err(\tilde{Y}^*) + \frac{4|\A|(\pi + |\A| -2)}{|\A| \pi -1} \sqrt{\frac{\log(32 |\A|/\delta)}{2n}}   \nonumber 
\end{flalign} 

\end{proof}

We are now ready for the proof of Theorem 1.

\noindent \paragraph{Theorem 1 }\textit{For any hypothesis class $\mathcal{H}$, any distribution over $(X,A,Y)$ and any $\delta \in (0,1/2)$, then with probability $1- \delta$},
if $n \ge \frac{16\log(8|\A|/\delta)}{\min_{ya}\P_{ya}}$, $\alpha_n =  \sqrt{\frac{8\log{64/\delta}}{n \min_{yz}\Q_{yz}}}$ and $\tilde{\alpha}_n =  \sqrt{\frac{\log(64/\delta)}{2n}} \frac{4 |\mathcal{A}|C^2}{\min_{ya}\P_{ya}^2}$ then  the predictor resulting from the two-step procedure satisfies:

\begin{flalign}
 \err(\tilde{Y})    \leq_{\delta} \err(Y^*)   +  \frac{5C}{\min_{ya} \mathbf{P}_{ya}^2}  \left( \frac{2}{B}  + 10 \mathfrak{R}_{\frac{\min_{ya}n \P_{ya}}{4}}(\mathcal{H}) + 18 |\mathcal{A}| \sqrt{\frac{2\log{64 |\A|/\delta}}{n \min_{ya}\P_{ya}}} \right)  \nonumber 
\end{flalign} 
\begin{equation*}
   \disc(\tilde{Y}) \le_{\delta}  \sqrt{\frac{ \log(\frac{64}{\delta})}{2n}} \frac{8 |\mathcal{A}|C^2}{\min_{ya}\P_{ya}^2} 
\end{equation*}

\begin{proof}
 Since the predictor obtained in step 1 is only a function of $X$, then the guarantees of step 2 immediately apply by Lemma \ref{lemma:step2-guarantees}:

\begin{flalign}
 \err(\tilde{Y})    \leq_{\delta/2} \err(\tilde{Y}^*) +   4 |\mathcal{A}|C \sqrt{\frac{\log(32 |\mathcal{A}|/\delta)}{2n}}  \nonumber 
\end{flalign} 
\begin{equation*}
   \disc(\tilde{Y}) \le_{\delta/2}  \sqrt{\frac{ \log(\frac{64}{\delta})}{2n}} \frac{8 |\mathcal{A}|C^2}{\min_{ya}\P_{ya}^2} 
\end{equation*}

Now we have to relate the loss of the optimal derived predictor from $\hat{Y}$, denoted by $\tilde{Y}^*$, to the loss of the optimal non-discriminatory predictor in $\mathcal{H}$. We can apply Lemma 4 in \cite{woodworth2017learning} as the solution of our derived LP is in expectation equal to that in terms of $A$. Lemma 4 in \cite{woodworth2017learning} tells us that the optimal derived predictor has a loss that is less or equal than the sum of the loss of the base predictor and its discrimination:

\begin{equation}
\err(\tilde{Y}^*) \leq \err(\hat{Y}) + \disc(\hat{Y})
\end{equation}

We have then by Lemma \ref{lemma:step2-guarantees} the loss of the optimal derived predictor:

\begin{align*}
    \err(\tilde{Y}^*) &\leq_{\delta}  \err(Y^*) + 4 \sqrt{\frac{\log{1/\delta}}{n}}  +4 \mathfrak{R}_{n/2}(\mathcal{H})+ \frac{5C}{\min_{ya} \mathbf{P}_{ya}^2}  \left( \frac{2}{B}  + 6 \mathfrak{R}_{\frac{\min_{ya}n \P_{ya}}{4}}(\mathcal{H}) + 10 \sqrt{\frac{2\log{64 |\A|/\delta}}{n \min_{ya}\P_{ya}}} \right)\\
    & \le_{\delta} \err(Y^*) + \frac{5C}{\min_{ya} \mathbf{P}_{ya}^2}  \left( \frac{2}{B}  + 10 \mathfrak{R}_{\frac{\min_{ya}n \P_{ya}}{4}}(\mathcal{H}) + 14 \sqrt{\frac{2\log{64 |\A|/\delta}}{n \min_{ya}\P_{ya}}} \right)
\end{align*} 
Hence our derived predictor satisfies:
\begin{flalign*}
 \err(\tilde{Y})    &\leq_{\delta/2} \err(\tilde{Y}^*) + 4 |\mathcal{A}|C \sqrt{\frac{\log(32 |\mathcal{A}|/\delta)}{2n}}    \nonumber  \\
 &\leq_{\delta} \err(Y^*)   +  \frac{5C}{\min_{ya} \mathbf{P}_{ya}^2}  \left( \frac{2}{B}  + 10 \mathfrak{R}_{\frac{\min_{ya}n \P_{ya}}{4}}(\mathcal{H}) + 18 |\mathcal{A}| \sqrt{\frac{2\log{64 |\A|/\delta}}{n \min_{ya}\P_{ya}}} \right)
\end{flalign*}

\end{proof}

\subsection{Section \ref{sec:missing}}

\noindent \textbf{Lemma 3} \textit{    Given a hypothesis class $\mathcal{H}$, a distribution over $(X,A,Y)$, $B \in \mathbb{R}^+$ and any $\delta \in (0,1/2)$, then with probability greater than $1- \delta$, if  $n_\ell \geq \frac{8\log{4 |\A|/\delta}}{\min_{ya} \P_{ya}}$, $\alpha_n = 2 \sqrt{\frac{\log{32|\A|/\delta}}{n_\ell \min_{ya}\P_{ya}}}$ and we let $\nu  = \mathfrak{R}_{n}(\mathcal{H}) + \sqrt{\frac{\log{4/\delta}}{n}}$, then running Algorithm 1 on data set $S$ with $T \geq \frac{16 \log(4 |\mathcal{A}|+1)}{\nu^2}$ and learning rate $\eta = \frac{\nu}{8 B}$ returns a predictor $\hat{Y}$ satisfying the following:}
	\begin{flalign*}
	\err(\hat{Y}) &\leq_{\delta} \err(Y^*)  + 4 \mathfrak{R}_{n}(\mathcal{H}) + 4 \sqrt{\frac{\log{4/\delta}}{n}} &\\
	disc(\hat{Y}) &\leq_{\delta}   \frac{2}{B}  + 6 \mathfrak{R}_{\frac{\min_{ya}n_\ell \P_{ya}}{2}}(\mathcal{H})  + 10 \sqrt{\frac{2\log{32 |\A|/\delta}}{n_\ell \min_{ya}\P_{ya}}} \ 
	\end{flalign*}

\begin{proof}
The proof follows immediately from Lemma \ref{lemma:step1_guarantees} with the identical error bound and replacing $n$ by $n_l$ in the discrimination bound. The two dataset langragian does not impact Theorem 1 in \cite{agarwal2018reductions} and the definition of an approximate saddle point remains the same as both players have the same objective. 
\end{proof}

\noindent \textbf{Lemma 4} \textit{	Let $S=\{(x_i,a_i,y_i)\}_{i=1}^{n}$ i.i.d. $\sim$ $\bP^n(A,X,Y)$,
	the estimator $\tilde{\gamma}^{S}_{ya}$ is consistent. As $n\to \infty$
	\begin{equation*}
	\tilde{\gamma}^{S}_{ya} \to_p  \gamma_{ya}.
	\end{equation*}
}
\begin{proof}
\begin{align*}
	\tilde{\gamma}_{ya}(\hat{Y})&=\lim_{n\to \infty} \frac{\frac{1}{n}\sum_{i=1}^n\hat{Y}(x_i)\mathbf{1}(y_i = y) \bP(A = a |x_i,y_i)}{\frac{1}{n} \sum_{i=1}^n \mathbf{1}(y_i = y)\bP(A = a |x_i,y_i)}\\
	&\to \frac{\bE[\hY(X) \bI(Y=y) \bP(A=a|X,Y)]}
	{\bE[ \bI(Y=y) \bP(A=a|X,Y) ]}\\
	&=\frac{\bE[\hY(X) \bI(Y=y) \bP(A=a|X,Y)]}
	{\int_{x} \bP(X=x,Y=y) \bP(A=a|X=x,Y=y) dx}\\
	&=\frac{\int_{x} \bP(X=x,Y=y) \hY(x)  \bP(A=a|X=x,Y=y) dx} {\bP(Y=y,A=a)}\\
	&=\frac{\int_{x} \bP(X=x|Y=y,A=a) \bP(Y=y,A=a)  \hY(x)   dx} {\bP(Y=y,A=a)}
	\\
	&=\bE_{X|Y=y,A=a} \hY(X) = \bP(\hY=1  |Y=y,A=a) = \gamma_{ya}
	\end{align*}

\end{proof}
\end{document}